\documentclass[12pt]{colt2016}

\usepackage{amsfonts,bbm,amsmath,color}
\usepackage{array, booktabs, multirow, calc}
\usepackage{microtype,times}
\usepackage{colortbl}

\newcommand{\R}{\mathbb{R}}

\newcommand{\X}{\mathcal{X}}
\newcommand{\err}{\mathrm{err}}

\renewcommand{\vec}[1]{\boldsymbol{#1}}
\newcommand{\E}{\mathop{\mathbb{E}}}

\newcommand{\Z}{\mathcal{Z}}

\newcommand{\Y}{\mathcal{Y}}

\newcommand{\Hyp}{\mathcal{H}}

\newcommand{\Prob}{\mathbb{P}}

\newcommand{\VCH}{\mathrm{VC}_{\Hyp}}

\definecolor{missing}{rgb}{.9,.9,.9}

\newenvironment{sproof}{\par\noindent{\bfseries\upshape
  Proof sketch\ }}{\hfill\BlackBox\\[2mm]}

\title{Minimax Lower Bounds for Realizable Transductive Classification}

\coltauthor{\Name{Ilya Tolstikhin}\\
  \Email{ilya@tuebingen.mpg.de}\\
  \addr  Max Planck Institute for Intelligent Systems, T\"ubingen, Germany
  \AND
  \Name{David Lopez-Paz}\\
  \Email{david@lopezpaz.org}\\
  \addr Facebook AI Research, Paris, France\thanks{A major part of this
  work was done while DLP was at the Max Planck Insitute for Intelligent
  Systems, T\"ubingen.}
}

\begin{document}

\maketitle

\thispagestyle{empty}

\begin{abstract}%
Transductive learning considers a {training set} of $m$ labeled samples and a
{test set} of $u$ unlabeled samples, with the goal of best labeling that
particular test set. 
Conversely, inductive learning considers a training set of $m$ labeled samples
drawn iid from $P(X,Y)$, with the goal of best labeling any future samples 
drawn iid from $P(X)$.
This comparison suggests that transduction is a much easier type of inference
than induction, but is this really the case?
This paper provides a negative answer to this question, by proving the first
known minimax lower bounds for transductive, realizable, binary classification.
Our lower bounds show that $m$ should be at least $\Omega(d/\epsilon +
\log(1/\delta)/\epsilon)$ when $\epsilon$-learning a concept class $\Hyp$ of
finite VC-dimension $d<\infty$ with confidence $1-\delta$, for all $m \leq u$. 
This result draws three important conclusions.
First, general transduction is as hard as general induction, 
since both problems have $\Omega(d/m)$ minimax values.
Second, the use of unlabeled data does not help general
transduction, since supervised learning algorithms such as ERM and \citep{H15}
match our transductive lower bounds while ignoring the unlabeled test set.
Third, our transductive lower bounds imply lower bounds for semi-supervised
learning, which add to the important discussion about the role of unlabeled data 
in machine learning.
\end{abstract}


\section{Introduction}
\emph{Transductive learning} \citep{Vap98} considers two sets of data: a
\emph{training set} containing $m$ labeled samples, and an \emph{unlabeled set}
containing $u$ unlabeled samples. Using these two sets, the goal of
transductive learning is to produce a classifier that best labels the $u$
samples in the unlabeled set. Transductive learning contrasts \emph{inductive
learning}, which is given $m$ labeled samples drawn iid from some probability
distribution $P(X,Y)$, and aims to produce a classifier that best labels any future
unlabeled samples drawn iid from $P(X)$. 

Transductive learning is a natural choice for learning problems where the
locations of the test samples are known at training time.  For instance,
consider the task of predicting where a particular person is named during one
thousand hours of speech. Because of time, financial, or technical constraints,
it may be feasible to manually label only a small fraction of the speech
frames, to be used as training set. Since the speech frames for both
training and test samples are known, this would be a learning problem well
suited for transduction.  More generally, transductive learning has found a
wide and diverse range of successful applications, including text
categorization, image colorization, image compression, image segmentation,
reconstruction of protein interaction networks, speech tagging, and statistical
machine translation; all of these discussed and referenced in \citep[Section
1.2]{P08}. For further discussions on transductive learning, see
\citep[Chapters 6, 24, 25]{ChaSchZie06}.

The previous paragraphs reveal that transduction is reasoning from known
training examples to known test examples, while induction is reasoning from
known training examples to unknown test examples. Such comparison suggests that
transduction is a much easier type of inference than induction. However, the
literature provides no rigorous mathematical justification for this statement.
The main contribution of this paper is to provide a negative answer. To this
end, we prove the first known minimax lower bounds on transductive, realizable,
binary classification when $m\leq u$. Our proofs are inspired by
their counterparts in inductive learning \citep{DGL96}, which rely on the worst
case analysis of binary classification and the probabilistic method.
Our results draw three important consequences.
First, we conclude that general transduction is as hard as general induction,
since both problems exhibit $\Omega(d/m)$ minimax values.
Second, we realize that the use of unlabeled data does not help general
transductive learning, since supervised learning algorithms such as empirical
risk minimization and the algorithm of \citet{H15} match our
transductive lower bounds while ignoring the unlabeled test set.
Third, we use our transductive lower bounds to derive lower bounds for
semi-supervised learning, and relate them to the impossibility results of
\citet{BD08} and \citet{scholkopf12anti}. Therefore, our results add to the
important discussion about the role of unlabeled data in machine learning.

The rest of this paper is organized as follows.
Section~\ref{sec:formal-definition} reviews the two settings of transductive
learning that we will study in this paper, and reviews prior literature
concerning their learning theoretical guarantees.
Section~\ref{sec:main-results} presents our main contribution: the first known
minimax lower bounds for transductive, realizable, binary classification.
Section~\ref{sec:consequences} discusses the consequences of our lower bounds.
Finally, Section~\ref{sec:conclusion} closes our exposition with a summary
about the state-of-affairs in the theory of transductive binary classification.
For future reference, Table~\ref{table:results} summarizes all the
contributions contained in this paper.

\begin{table}[ht!]
\centering
\resizebox{\linewidth}{!}{
\renewcommand{\arraystretch}{1.2}
\begin{tabular}{|l|l|c|c|}
  \cline{3-4}
  \multicolumn{2}{c|}{} & Transductive S. I & Transductive S. II \\ \hline
  \multirow{4}{*}{ERM upper bound}
    & \multirow{2}{*}{with probability $\geq 1-\delta$}
      & $O\left(\frac{\VCH\log(N) + \log\frac{1}{\delta}}{m}\right)$
      & $O\left(\frac{\VCH\log(m) + \log\frac{1}{\delta}}{m}\right)$\\
      && {\small Theorem~\ref{thm:erm-upper-setting1}}
      &  {\small Theorem~\ref{thm:erm-upper-setting2-better}} \\\cline{2-4}
    & \multirow{2}{*}{in expectation}
      & $O\left(\frac{\VCH\log(N)}{m}\right)$
      & $O\left(\frac{\VCH\log(m)}{m}\right)$\\
      && {\small Theorem~\ref{thm:erm-upper-setting1}}
      &  {\small Theorem~\ref{thm:erm-upper-expect-setting2}}\\\cline{1-4}
      \multirow{4}{*}{\cite{H15} upper bound}
    & \multirow{2}{*}{with probability $\geq 1-\delta$}
      & \multirow{2}{*}{---} 
      & $O\left(\frac{\VCH + \log\frac{1}{\delta}}{m}\right)$\\
      && 
      &  {\small Theorem~\ref{thm:erm-upper-setting2-better}} \\\cline{2-4}
    & \multirow{2}{*}{in expectation}
      & \multirow{2}{*}{---} 
      & $O\left(\frac{\VCH}{m}\right)$\\
      && 
      &  {\small Theorem~\ref{thm:erm-upper-expect-setting2}}\\\cline{1-4}
  \multirow{4}{*}{Minimax lower bound}
    & \multirow{2}{*}{in probability}
      &  $\Omega\left( \frac{\VCH + \log\frac{1}{\delta}}{m}\right)$
      &  $\Omega\left( \frac{\VCH + \log\frac{1}{\delta}}{m}\right)$ \\
      && {\small Corollary~\ref{cor:sample-complexity-setting1}}
      &  {\small Corollary~\ref{cor:sample-complexity-setting2}}\\\cline{2-4}
    & \multirow{2}{*}{in expectation}
      &  $\Omega\left( \frac{\VCH}{m}\right)$
      &  $\Omega\left( \frac{\VCH}{m}\right)$\\
      && {\small Theorem~\ref{thm:minimax_exp}}
      &  {\small Theorem~\ref{thm:lower-exp-setting2}} \\\cline{1-4}
  \multicolumn{2}{|c|}{ERM gap} & $O(\log N )$  & $O(\log m)$ \\\hline
  \multicolumn{2}{|c|}{\cite{H15} gap} & ---
  & $O(1)$ \\\hline
\end{tabular}
}
\caption{Upper and lower bounds for transductive, realizable, binary
classification.  All the results are original contributions, except for
Theorem~\ref{thm:erm-upper-setting1}.}
\label{table:results}
\end{table}

\section{Formal problem definition and assumptions}\label{sec:formal-definition}
Transductive learning algorithms are given a \emph{training set}\footnote{The
sets presented in this paper are treated as \emph{ordered multisets}.} $\Z_m :=
\{(X_i,Y_i)\}_{i=1}^m{\subseteq \X\times \{0,1\}}$ and an \emph{unlabeled set}
$\X_u :=\{X_i\}_{i=m+1}^{{m+u}}{\subseteq \X}$, where {$\X$ is an input space}.
Here, the unlabeled set is constructed from some unknown \emph{test set} $\Z_u
:= \{(X_i,Y_i)\}_{i=m+1}^{{m+u}}$, that is, $\X_u = \{X : (X,Y) \in \Z_u\}$.
Given a set $\Hyp$ of classifiers mapping $\X$ to $\{0,1\}$, the training set
$\Z_m$, and the unlabeled set~$\X_u$, the goal of transductive learning is to
choose a function $h_m=h_m(\Z_m, \X_u)\in \Hyp$ which best predicts labels for
the unlabeled set $\X_u$, as measured by
\[
\err(h_m, \Z_u) :=
\frac{1}{u}\sum_{(x,y)\in \Z_u} \mathbbm{1}\{h_m(x)\neq y\}.
\]

In this paper, we analyze the two settings of transductive learning proposed by
\cite{Vap98}:
\begin{itemize}
  \item Setting I (TLSI) assumes a fixed {population set} $\Z_N :=
  \{(X_i,Y_i)\}_{i=1}^N\subseteq \X \times \{0,1\}$ {with $N:=m+u$}. By
  sampling uniformly without replacement from $\Z_N$, we construct the
  \emph{training set} $\Z_m$, of size $m$. The remaining $u$ data
  points form the \emph{test set} $\Z_u = \Z_N \setminus \Z_m$.
  \item Setting II (TLSII) assumes a fixed probability distribution
  $P$ on $\X\times\{0,1\}$. The \emph{training set}
  $\Z_m:=\{(X_i,Y_i)\}_{i=1}^m$ and the \emph{test set}
  $\Z_u:=\{(X_j,Y_j)\}_{j=m+1}^{{m+u}}$ are sets of independently and
  identically distributed (iid) samples from $P$. 
\end{itemize}
In both settings, the \emph{unlabeled set} is $\X_u := \{ X : (X,Y) \in \Z_u
\}$. Table~\ref{table:settings} summarizes the differences between TLSI and
TLSII, when compared together with inductive supervised learning \citep{Vap98},
denoted by SL, and inductive semi-supervised learning \citep{ChaSchZie06},
denoted by SSL. Two facts of interest arise from this comparison. First, TLSII
and SSL differ only on their objective: while TLSII minimizes the
classification error over the {given} unlabeled set $\X_u$, SSL minimizes the
classification error over the entire marginal distribution $P(X)$. Second, TLSI
provides learners with more information than TLSII. This is because all
the randomness in TLSI is due to the partition of the population set $\Z_N$. Thus, in
TLSI the entire marginal distribution $P(X)$ is known to the learner, and the
only information missing from the joint distribution $P(X,Y)$ are the $u$
binary labels missing from the unlabeled set $\X_u$. This is in contrast to TLSII, where
the learner faces a partially unknown marginal distribution $P(X)$.

\begin{table}[h!]
  \resizebox{\linewidth}{!}{
  \renewcommand{\arraystretch}{1.4}  
  \begin{tabular}{|l|c|c|c|c|}
    \cline{2-5}
    \multicolumn{1}{c|}{} & Transductive S. I & Transductive S. II & Semi-Supervised & Supervised\\\hline
    \multirow{2}{*}{Training set $\Z_m$}
      & $\Z_m$ sampled uniformly without
      & \multicolumn{3}{c|}{\multirow{2}{*}{$\Z_m \stackrel{\text{i.i.d.}}{\sim} P(X,Y)$}}\\ 
      & replacement from $\Z_N :=\{(X_i, Y_i)\}_{i=1}^N$
      & \multicolumn{3}{c|}{} \\ \cline{1-5}
    Unlabeled set $\X_u$
      & inputs from $\Z_u := \Z_N \setminus \Z_u$
      & \multicolumn{2}{c|}{$\X_u \stackrel{\text{i.i.d.}}{\sim} P(X)$} &$\X_u = \{\emptyset\}$\\\hline
    Choose $h_m$ minimizing
      & \multicolumn{2}{c}{$\err(h_m, \Z_u)$}
      & \multicolumn{2}{|c|}{$\Prob_{(X,Y)\sim P}\{h_m(X)\neq Y\}$}\\\hline
  \end{tabular}
  }
  \label{table:settings}
  \caption{Learning settings and their objectives.}
\end{table}

\paragraph{Assumptions} Our analysis calls for three assumptions. First, we
assume a finite VC-dimension for $\Hyp$. Second, we assume
\emph{realizability}, that is, the existence of a function $h^\star \in \Hyp$
such that $h^\star(x) = y$ for all $(x,y) \in \Z_N$ in TLSI, or $h^\star(X) =
Y$ with probability 1 for all $(X,Y) \sim P$ in TLSII. Third, we assume $m \leq
u$.  The first two assumptions are commonly used throughout the literature in
learning theory \citep{DGL96, Vap98, SS14}. Although in some situations
restrictive, these assumptions ease the analysis of the first known minimax
lower bounds for transductive classification. The third assumption is natural,
since unlabeled data is cheaper to obtain than labeled data.

\subsection{Prior art}

The literature in learning theory provides a rich collection of upper bounds on
the learning rates for TLSI.  \cite{Vap82,Vap98} provides sharp upper bounds for
Empirical Risk Minimization (ERM) in TLSI.  However, these ERM upper bounds are
only explicit for the $m=u$ case.  To amend this issue, \cite{CM06} extend
these bounds to the $m\neq u$ case.  In particular, this results in an upper
bound for ERM in TLSI of the order $\VCH \log(m+u)/m$, where $\VCH$ is the
VC-dimension of the learning hypothesis class $\Hyp$. Following a different approach, 
\citet{BL03} provide upper bounds depending on the hypothesis
class {prior distribution}.  Under realizability assumptions and good choices
of hypothesis class prior distributions, these bounds lead to fast $m^{-1}$
learning rates.  Most recently, \cite{TBK14} provide general bounds which
achieve fast rates $o(m^{-1/2})$ under Tsybakov low noise assumptions,
recovering the $\VCH \log(m+u)/m$ upper bound of \cite{CM06} with looser
constants. Regarding TLSII, upper bounds are usually obtained from the
corresponding upper bounds in TLSI \cite[Theorem 8.1]{Vap98}. However, this
strategy is in many cases suboptimal, as we will later address in
Section~\ref{sec:consequences}. 

Notably, the literature does not provide with lower bounds for either TLSI or
TLSII. The following section addresses this issue, by providing the first known
minimax lower bounds for transductive, realizable, binary classification.

\section{Main results}\label{sec:main-results}

This section develops lower bounds for the \emph{minimax probability of error} 
\[
\inf_{h_m}\sup \Prob\left\{ \err(h_m, \Z_u) -
\inf_{h\in \Hyp}
\err(h, \Z_u) \geq \epsilon\right\}
\]
and the \emph{minimax expected risk}
\[
\inf_{h_m}\sup \E\left[ \err(h_m, \Z_u) -
\inf_{h\in \Hyp}
\err(h, \Z_u)\right]
\]
of transductive learning algorithms $h_m$.  In the previous, the suprema are taken over
all possible realizable distributions of training sets $\Z_m$ and test sets $\Z_u$,
and the outer infima are taken over all
transductive learning algorithms $h_m = h_m(\Z_m, \X_u)$.
Finding a lower bound to these values guarantees, for every
possible transductive learning algorithm $h_m$, the existence of learning
problems which cannot be solved by $h_m$ faster than at a certain learning
rate. This is the goal of the rest of this section.

Our proofs are inspired by their analogous in the classical setting (inductive
and iid) of statistical learning theory \citep{Vap98}.  In particular, our
arguments involve standard constructions based on $\VCH$ points shattered by
$\Hyp$ and the use of {the probabilistic method} \citep{DGL96}.
However, due to the combinatorial (sampling without replacement) nature of TLSI,
we had to develop new arguments to apply these techniques to our problem.
Remarkably, the rates of our lower bounds are almost identical to the ones from
the classical setting of statistical learning theory \cite[Section 14]{DGL96},
which shows that general transduction is as hard as general induction. In the
following, we will proceed separately for TLSI and TLSII.

\subsection{Minimax lower bounds for TLSI}\label{subsection:lower-setting-1}

Consider the minimax probability of error
\begin{align*}
\mathcal{M}_{\epsilon,N,m}^{\mathrm{I}}(\Hyp)
&:=
\inf_{h_m}
\sup_{\Z_N}
\Prob \left\{
\err(h_m, \Z_u)
-
\inf_{h\in \Hyp}
\err(h, \Z_u)
\geq \epsilon
\right\},
\end{align*}
where the outer infimum runs over all transductive learning algorithms $h_m =
h_m(\Z_m,\X_u)$ based on the training set $\Z_m$ and the unlabeled set $\X_u$
built as in TLSI, and the supremum runs over all possible population sets
$\Z_N$ realizable by $\Hyp$. Then, the following result lower bounds
$\mathcal{M}_{\epsilon,N,m}^{\mathrm{I}}(\Hyp)$.

\begin{theorem}
\label{thm:minimax_hp}
Consider TLSI.
Let $\Hyp$ be a set of classifiers with VC dimension $2\leq d<\infty$.
Assume the existence of $h^{\star} \in \Hyp$, such that $h^{\star}(x) = y$ for all $(x,y)\in \Z_N$.
\begin{enumerate}
\item
\label{item:Theorem-1}
If 
$u\geq m \geq 8(d-1)$ and $\epsilon \leq 1/32$, then
\[
\mathcal{M}_{\epsilon,N,m}^{\mathrm{I}}(\Hyp)
\geq
\frac{1}{150}e^{-32 m \epsilon}.
\]
If $d < 7$ the constant $1/150$ can be improved to $1/4$.
\item
\label{item:Theorem-2}
If $\max\{9, 2(d-1)\} \leq m \leq \min\{ d/(24\epsilon), u\}$, then
\[
\mathcal{M}_{\epsilon,N,m}^{\mathrm{I}}(\Hyp)
\geq
\frac{1}{16}.
\]
\end{enumerate}
\end{theorem}
\begin{sproof}
The full proof is provided in Appendix~\ref{proof:minimax_hp}. The proofs of
Theorems~\ref{thm:minimax_exp}, \ref{thm:minimax_hp-II}, and
\ref{thm:lower-exp-setting2} follow a similar sketch.

\textit{Step 1, restriction to particular $\Z_N$.}
Due to the realizability, $\inf_{h\in \Hyp} \err(h, \Z_u)$ vanishes, and
\[
\mathcal{M}_{\epsilon,N,m}^{\mathrm{I}}(\Hyp)
=
\inf_{h_m}
\sup_{\Z_N}
\Prob \left\{
\err(h_m, \Z_u)
\geq \epsilon
\right\}.
\]
Next, we lower bound the previous expression by running the supremum over some
particular family of population sets $\Z_N$.
First, select $d$ distinct points $\{x_1,\dots,
x_d\}\subseteq \X$ shattered by $\Hyp$.
Second, let $\vec b:=(b_1,\dots,b_d)$ be any binary string, and let $\vec i:=
(i_1,\dots, i_d)$ be any sequence of nonnegative integers such that
$\sum_{j=1}^{d} i_j = N$.
Third, let the vectors $\vec b$ and $\vec i$ parametrize a family of population
sets $\Z_N$, where the set $\Z_N(\vec b, \vec i)$ contains $i_j \geq 0$ copies
of $(x_j, b_j)$ for all $j=1,\dots,d$.
Clearly, every such $\Z_N(\vec b, \vec i)$ satisfies the realizability assumption.
Let $\vec K:=(K_1,\dots, K_d)$, where $K_j$ is the number of copies
(multiplicity) of the input $x_j$ contained in the random test set $\Z_u$.
Then,
\[
\mathcal{M}_{\epsilon,N,m}^{\mathrm{I}}(\Hyp)
\geq
\inf_{h_m}
\sup_{\vec b, \vec i}
\Prob \left\{
\frac{1}{u}\sum_{j = 1}^d K_j \mathbbm{1}\{h_m(x_j)\neq b_j\}
\geq \epsilon
\right\}.
\]

\textit{Step 2, use of the probabilistic method.}
The supremum over the binary string $\vec b$ can be lower bounded by the
expected value of a random variable $B$ uniformly distributed over $\{0,1\}^d$.
Then,
\[
\mathcal{M}_{\epsilon,N,m}^{\mathrm{I}}(\Hyp)
\geq
\inf_{h_m}
\sup_{\vec i}
\Prob \left\{
\frac{1}{u}\sum_{j = 1}^d K_j \mathbbm{1}\{h_m(x_j)\neq B_j\}
\geq \epsilon
\right\}.
\]
We can further lower bound the previous expression as 
\begin{equation}
\label{eq:proof-sketch-step-2}
\mathcal{M}_{\epsilon,N,m}^{\mathrm{I}}(\Hyp)
\geq
\inf_{h_m}
\sup_{\vec i}
\Prob \left\{
\sum_{j = 1}^d i_j \mathbbm{1}\{K_j = i_j\} \mathbbm{1}\{h_m(x_j)\neq B_j\}
\geq u\epsilon
\right\}.
\end{equation}

\textit{Step 3, lower bounding tails of binomial and hypergeometric distributions.}
If $K_j = i_j$ holds for
some $j\in\{1,\dots, d\}$, then the input $x_j$ did not appear in the training
set $\Z_m$.  In other words, the learning algorithm $h_m$ did not see the
output $B_j$ and, consequently, $h_m(x_j)$ is statistically independent of
$B_j$, and thus $\mathbbm{1}\{h_m(x_j)\neq
B_j\}\sim \mathrm{Bernoulli}(0.5)$.  Moreover, if $K_{j_1} = i_{j_1}$ and
$K_{j_2} = i_{j_2}$ for $j_1\neq j_2$ then $\mathbbm{1}\{h_m(x_{j_1})\neq
B_{j_1}\}$ and $\mathbbm{1}\{h_m(x_{j_2})\neq B_{j_2}\}$ are statistically
independent.  This shows that, when conditioning on $\vec K$, the sum in
\eqref{eq:proof-sketch-step-2} follows a Binomial distribution with parameters
$\bigl(\sum_{j = 1}^d \mathbbm{1}\{K_j = i_j\},0.5\bigr)$.  Finally, we observe
that the vector $\vec K$ follows a \emph{hypergeometric distribution}.  We
conclude by lower bounding the tails of Binomial and hypergeometric
distributions using the Chebyshev-Cantelli inequality \cite[Theorem
A.17]{DGL96} and other tools of probability theory.
\end{sproof}

Theorem~\ref{thm:minimax_hp} can be translated into a lower bound on the
\emph{sample complexity} of TLSI.  As the following result highlights, any
transductive learning algorithm $h_m$ needs at least $\Omega\bigl((\VCH - \log
\delta)/\epsilon\bigr)$ labeled points to achieve $\epsilon$ accuracy with
$\delta$ confidence for all configurations of realizable population sets
$\Z_N$.

\begin{corollary}
\label{cor:sample-complexity-setting1}
Consider the assumptions of Theorem \ref{thm:minimax_hp}.
Assume $0<\epsilon \leq 1/32$, $0<\delta\leq1/150$, and $\max\{9,8(d-1)\}\leq m \leq \min\{ d/(24\epsilon), u\}$. 
Let $C>0$ be a universal constant, and let the number of labeled samples satisfy 
\[
m <
C\left(
\frac{d}{\epsilon} + \frac{1}{\epsilon}\log\frac{1}{\delta}
\right).
\]
Then, any transductive learning algorithm $h_m$ satisfies
\begin{equation}
\label{eq:sample-complexity}
\sup_{\Z_N}\Prob\left\{
\err(h_m, \Z_u) \geq \epsilon
\right\}
\geq \delta.
\end{equation}
\end{corollary}
\begin{proof}
Due to Theorem~\ref{thm:minimax_hp}, if $m\leq
\frac{1}{32\epsilon}\log\frac{1}{150\delta}$ (c.f.
Statement~\ref{item:Theorem-1}) or $m \leq \frac{d}{24\epsilon}$ (c.f.
Statement~\ref{item:Theorem-2}), then the minimax probability is lower bounded
by $\delta$.
In other words, Equation~\ref{eq:sample-complexity} holds if
$m\leq\max\bigl\{ \frac{1}{32\epsilon}\log\frac{1}{150\delta}, \frac{d}{24\epsilon}\bigr\}$.
We conclude by writing
\[
\max\left\{ \frac{1}{32\epsilon}\log\frac{1}{150\delta}, \frac{d}{24\epsilon}\right\}
\geq
\frac{1}{64\epsilon}\log\frac{1}{150\delta} + \frac{d}{48\epsilon}
=
\Theta\left(
\frac{d}{\epsilon} + \frac{1}{\epsilon}\log\frac{1}{\delta}
\right).
\]
\end{proof}

The previous results hold in high probability. Next, we lower bound
the minimax expected risk
\begin{align*}
\mathcal{M}_{N,m}^{\mathrm{I}}(\Hyp)
&:=
\inf_{h_m}
\sup_{\Z_N}
\E
\left[
	\err(h_m, \Z_u)
	-
	\inf_{h\in \Hyp}
	\err(h, \Z_u)
\right].
\end{align*}
\begin{theorem}
\label{thm:minimax_exp}
Consider TLSI.  Let $\Hyp$ be a set of classifiers with VC dimension $2\leq
d<\infty$.  Assume the existence of $h^{\star} \in \Hyp$, such that
$h^{\star}(x) = y$ for all $(x,y)\in \Z_N$.  If~$\max\{9, d-1\} \leq m \leq u$,
then
\[
\mathcal{M}_{N,m}^{\mathrm{I}}(\Hyp)
\geq
\frac{d - 1}{16m}.
\]
\end{theorem}
\begin{proof}
See Appendix \ref{section:minimax_exp}.
\end{proof}

\subsection{Minimax lower bounds for TLSII}
We start the analysis of TLSII by lower bounding the minimax probability 
\[
\mathcal{M}^{\mathrm{II}}_{\epsilon,N,m}(\Hyp)
:=
\inf_{h_m}
\sup_{P}
\Prob\left\{
\err(h_m,\Z_u)
-
\inf_{h\in \Hyp}\err(h,\Z_u)
\geq \epsilon
\right\},
\]
where the supremum runs over all probability distributions $P$ realizable by $\Hyp$. 
\begin{theorem}
\label{thm:minimax_hp-II}
Consider TLSII.  Let $\Hyp$ be a set of classifiers with VC dimension $2\leq
d<\infty$.  Assume the existence of $h^{\star} \in \Hyp$, such that
$h^{\star}(X) = Y$ with probability 1 for $(X,Y)\sim P$.
\begin{enumerate}
\item
\label{item:st1-s2}
If $\max\{(d - 1)/2,10\}\leq m \leq \min\{(d-1)/(21\epsilon), u\}$, then
\[
\mathcal{M}^{\mathrm{II}}_{\epsilon,N,m}(\Hyp)
\geq
\frac{1}{80}.
\]
\item
\label{item:st2-s2}
If $0 < \epsilon\leq 1/32$ and $m\geq d - 1$, then
\[
\mathcal{M}^{\mathrm{II}}_{\epsilon,N,m}(\Hyp)\geq
\frac{1}{18} e^{-32m\epsilon}.
\]
\end{enumerate}
\end{theorem}
\begin{proof}
See Appendix~\ref{sect:proof-minimax_hp-II}.
\end{proof}

Theorem~\ref{thm:minimax_hp-II} can be translated into a lower bound on the sample complexity of TLSII.

\begin{corollary}
\label{cor:sample-complexity-setting2}
Consider setting of Theorem \ref{thm:minimax_hp-II}.
Assume $0<\epsilon \leq 1/32$, $0<\delta\leq1/80$, and $\max\{d - 1,10\}\leq m \leq \min\{(d-1)/(21\epsilon), u\}$. 
Let $C>0$ be an universal constant, and let the number of labeled examples satisfy
\[
m <
C\left(
\frac{d}{\epsilon} + \frac{1}{\epsilon}\log\frac{1}{\delta}
\right).
\]
Then, any transductive learning algorithm $h_m$ satisfies
\[
\sup_{P}\Prob\left\{
\err(h_m, \Z_u) \geq \epsilon
\right\}
\geq \delta.
\]
\end{corollary}
\begin{proof}
The proof is analogous to the one of Corollary~\ref{cor:sample-complexity-setting1}.
\end{proof}

Finally, we provide a lower bound on the minimax expected risk of TLSII, defined as 
\begin{align}
\label{eq:second_line}
\mathcal{M}^{\mathrm{II}}_{N,m}(\Hyp)
:=
\inf_{h_m}
\sup_{P}
\E\left[
\err(h_m,\Z_u)
-
\inf_{h\in \Hyp}\err(h,\Z_u)
\right].
\end{align}

\begin{theorem}
\label{thm:lower-exp-setting2}
Consider TLSII.  Let $\Hyp$ be a set of classifiers with VC-dimension $2\leq
d<\infty$.  Assume the existence of $h^{\star} \in \Hyp$, such that
$h^{\star}(X) = Y$ almost surely for $(X,Y)\sim P$.
If $d-1 \leq m$, then 
\[
\mathcal{M}^{\mathrm{II}}_{N,m}(\Hyp)
\geq
\frac{d-1}{2em}\left(1 - \frac{1}{m}\right).
\]
\end{theorem}
\begin{proof}
See Appendix~\ref{section:lower-exp-setting2}.
\end{proof}

\section{Consequences of main results}\label{sec:consequences}

This section describes three important consequences of the results presented in
Section~\ref{sec:main-results}.

\subsection{General transductive learning is as hard as general inductive learning}

First, general transduction is as hard as general induction, since the minimax
values of these two problems have the same order.  Said differently, in order
to find the $\epsilon$-best predictor in the class $\Hyp$ with high probability
simultaneously over all data-generating distributions, $\Omega(\VCH/\epsilon)$
labeled points are necessary for both transductive and inductive learning.

\subsection{Unlabeled data are not of significant help in general transductive learning}
Second, we show that using the unlabeled set $\X_u$ when training a
transductive learning algorithm does not bring a significant benefit in the
absence of additional assumptions.

To this end, we will compare transductive learning algorithms against two
supervised learning algorithms. First, Empirical Risk Minimization or ERM
\citep{Vap98}, denoted by $\hat{h}_m$. Second, the majority voting ensemble of
ERMs trained on subsets of $\Z_m$ proposed by \citet{H15}, denoted by
$\tilde{h}_m$. The goal of this section is to show that $\hat{h}_m$ and
$\tilde{h}_m$ achieve almost minimax optimal rates in both TLSI and TLSII. For
TLSI, Theorem~\ref{thm:erm-upper-setting1} will show that $\hat{h}_m$ achieves
the TLSI lower bounds of Theorems~\ref{thm:minimax_hp} and
\ref{thm:minimax_exp} up to $\log N$ factors.  For TLSII,
Theorem~\ref{thm:erm-upper-setting2-better} will show that $\tilde{h}_m$
achieves the TLSII lower bounds of Theorems~\ref{thm:minimax_hp-II} and
\ref{thm:lower-exp-setting2} up to \emph{constant} factors, and that
$\hat{h}_m$ achieves the same lower bounds up to $\log m$ factors. Since both
$\hat{h}_m$ and $\tilde{h}_m$ ignore the unlabeled set when solving
transduction, such results that unlabeled data is not of significant help in
general transductive learning.

\subsubsection{Unlabeled data in TLSI}\label{sec:s1-upper}

The following result upper bounds the risk of $\hat{h}_m$ in TLSI.  The
argument is a slight modification of \cite[Corollary 1]{CM06}, and also follows
from \cite[Corollary 14]{TBK14}.

\begin{theorem}
\label{thm:erm-upper-setting1}
Consider TLSI.  Let $\Hyp$ be a set of classifiers with VC-dimension $2\leq
d<\infty$ and assume the existence of $h^{\star} \in \Hyp$, such that
$h^{\star}(x) = y$ for all $(x,y)\in \Z_N$.  Assume that $u \geq 4$ and $u \geq
m \geq d - 1$.  Then, for any $\delta\in(0,1)$ and with probability at least
$1-\delta$ over the random choices of samples $\Z_m$ and $\Z_u$, the following
upper bound holds for $\hat{h}_m$:
\[
\err(\hat{h}_m, \Z_u)
\leq
\frac{2(d\log(Ne/d) + \log\frac{1}{\delta})}{m}.
\]
An integration of this upper bound also leads to 
\[
\E\left[\err(\hat{h}_m, \Z_u)\right]
\leq
\frac{2d\log(Ne/d) + 2}{m},
\]
where the expectation is taken with respect to the training sample $\Z_m$ and
the test sample $\Z_u$.
\end{theorem}
\begin{proof}
  See Appendix~\ref{proof:erm-upper-setting1}.
\end{proof}
Together with Theorems \ref{thm:minimax_hp} and \ref{thm:minimax_exp},
Theorem~\ref{thm:erm-upper-setting1} shows that empirical risk minimization
achieves the minimax optimal rate for TLSI up to $\log N$ gap.

Can this $\log N$ gap be improved to $\log m$?
There is hope in some situations.
First, if $m = \alpha N$ with $0<\alpha\leq 1/2$, then this improvement is
possible, since $\log N = \log (m/\alpha)$.
Second, if $m \ll N$, then the uniform sampling without replacement of $\Z_m$
from $\Z_N$ approaches the sampling with replacement (iid) of $\Z_m$ from
$\Z_N$, since it is unlikely that the same object from $\Z_N$ will appear in
$\Z_m$ more than once. \cite{DF80} precise this intuition, by showing that
the total variation distance between these two distributions (the one due to
sampling without replacement versus the one due to sampling with replacement)
is bounded between $1-e^{-\frac{1}{2}m(m-1)/N}$ and $\frac{1}{2}m(m-1)/N$.
Said differently, assuming $m = o(\sqrt{N})$ morphs TLSI into iid learning as
$N\to \infty$.  In such limit case, the upper bound of ERM falls back to
$O\left(m^{-1}({d \log m + \log(1/\delta)})\right)$ \citep{Vap98}.  However, we
lack any intuition if the gap can be improved when $m$ is between $\Omega(\sqrt{N})$
and $o(N)$.

A second question is whether the $\log N$ factor in Theorem
\ref{thm:erm-upper-setting1} could be avoided altogether.  The analogous
question in the iid setting served over twenty five years of research, where
the series of works \citep{BEH+89, EHK+89, DL95} proved minimax lower bounds of
the order $\Omega\bigl(\frac{d + \log(1/\delta)}{m}\bigr)$.
At the same time, \cite{AO07} showed that the upper bound
$O\left(m^{-1}({d\log(m) + \log(1/\delta)})\right)$ is not improvable for ERM.
Only recently it was finally proved by \cite{H15} that $O(d/m)$ rate is achieved by a majority voting supervised algorithm.
Unfortunately, the counterexample from \cite{AO07} does not apply to TLSI.
This is because their argument used the fact that to observe $n - d$ different
values of a uniform random variable taking $n$ values, it is necessary to
sample it at least $\Omega\bigl(n\log(n/d)\bigr)$ times. While this is true for
sampling {with replacement} (the same values can be observed repeatedly), the
claim does not follow for the sampling {without replacement} employed TLSII. 

\subsubsection{Unlabeled data in TLSII}
Consider any supervised algorithm $h_m^0$ which ignores the unlabeled set
$\X_u$. Then, 
\[
\E\left[\err(h^0_m, \Z_u)\right]
=
\E_{\Z_m}\left[\Prob_{(X,Y)\sim P}\left\{
h^0_m(X) \neq Y
\right\}\right]
\]
The right hand side of the previous equality is the expected error probability
of the learning algorithm $h^0_m$ under the standard iid setting of supervised
classification.
Therefore, upper bounds on the expected test error of $h_m^0$ in TLSII
follow from upper bounds of the standard iid setting of supervised learning. In
particular, the following result is a direct consequence of \cite[Problems
12.8]{DGL96} and \cite[Theorem 2]{H15}.

\begin{theorem} \label{thm:erm-upper-expect-setting2} Consider TLSII. Let
$\Hyp$ be a set of classifiers with VC dimension $d<\infty$.  Assume the
existence of $h^{\star} \in \Hyp$, such that $h^{\star}(X) = Y$ with
probability 1 for $(X,Y)\sim P$.  Let $\hat{h}_m$ be the ERM, and let
$\tilde{h}_m$ be the algorithm of \citep{H15}. Then,
\[
\E\left[\err(\hat{h}_m, \Z_u)\right] 
\leq \frac{2d\log(2m) + 4}{m\log(2)}
\]
and
\[
\E\left[\err(\tilde{h}_m, \Z_u)\right] 
\leq 
O\left(\frac{d}{m}\right).
\]
All the expectations are taken with respect to both the training set $\Z_m$ and
the test set $\Z_u$.  These bounds hold for unlabeled sets of all sizes.
\end{theorem}

It is well known that upper bounds for TLSI lead to upper bounds for TLSII
\cite[Theorem 8.1]{Vap98}.  Over the years, researchers have derived upper
bounds for TLSII using upper bounds from TLSI (for a detailed discussion, see
Appendix~\ref{subsec:II-to-I}).  However, this approach leads in many cases to
suboptimal upper bounds for TLSII. Instead, we now derive sharper upper bounds
for TLSII using a direct analysis. 

\begin{theorem}
\label{thm:erm-upper-setting2-better}
Consider TLSII.  Let $\Hyp$ be a set of classifiers with VC dimension
$d<\infty$.  Assume the existence of $h^{\star} \in \Hyp$, such that
$h^{\star}(X) = Y$ with probability 1 for $(X,Y)\sim P$.  Then for any $\delta
\in (0,1)$ with probability at least $1-\delta$ (over the random choices of
samples $\Z_m$ and $\Z_u$) for ERM $\hat{h}_m$ it holds that
\[
\err(\hat{h}_m, \Z_u)
\leq
\frac{6d\log(m) + 3\log\frac{2}{\delta} + 3\log 2}{2m}
+
\frac{5\log\frac{2}{\delta}}{3u}
\stackrel{(\star)}{=}
O\left(
\frac{d\log(m) + \log\frac{1}{\delta}}{m}
\right)
\]
and for the algorithm $\tilde{h}_m$ of \cite{H15} it holds that
\[
\err(\tilde{h}_m, \Z_u)
\leq
O\left(\frac{d + \log\frac{2}{\delta}}{m}\right)
+
\frac{5\log\frac{2}{\delta}}{3u}
\stackrel{(\star)}{=}
O\left(
\frac{d + \log\frac{1}{\delta}}{m}
\right),
\]
where ($\star$) holds if $u \geq m$.
\end{theorem}
\begin{sproof} 
  For a full proof, see Appendix~\ref{proof:erm-upper-setting2-better}.
  
  If the learning algorithm $h^0_m$ does not use the unlabeled set $\X_u$
  during its training then, when conditioning on the training set $\Z_m$, the
  test error $\err(h^0_m,\Z_u)$ follows the distribution of an average of $u$
  iid Bernoulli random variables with parameters $L(h^0_m) := \Prob_{(X,Y)\sim
  P}\{ h^0_m(X) \neq Y\}$.
  We bound this average by using Bernstein's inequality \cite[Theorem
  2.10]{BLM13} and accounting for the fact that $L(\hat{h}_m)$ and
  $L(\tilde{h}_m)$ can be upper bounded with high probability by using
  \cite[Problem 12.9]{DGL96} and \cite[Theorem 2]{H15}.
\end{sproof}

Together with Theorems \ref{thm:minimax_hp-II} and
\ref{thm:lower-exp-setting2}, Theorems~\ref{thm:erm-upper-expect-setting2} and
\ref{thm:erm-upper-setting2-better} show that $\tilde{h}_m$ is one optimal
learning algorithm for TLSII.

\subsection{Lower bounds on TLSII lead to lower bounds on supervised and
semi-supervised learning} \label{subsection:relations}

Third, our lower bounds shed light on the relationships between the minimax
values of TLSII, supervised learning, and semi-supervised learning.  In the
following, let $h_m$ be a learning algorithm with access to the training set
$\Z_m$ and the unlabeled set $\X_u$, and let $h_m^0$ be a learning algorithm
with access only to the training set $\Z_m$.  For any $h\in \Hyp$, we denote by
$L(h):=\Prob_{(X,Y)\sim P}\{ h(X) \neq Y\}$ the error probability of $h$.

We start by observing that, under realizability, $\inf_{h\in \Hyp} L(h) = 0$.
Then, for any supervised learning algorithm $h_m^0$, we define its minimax
probability of error 
\[
\mathcal{M}^{\mathrm{SL}}_{\epsilon,m}(\Hyp)
:=
\inf_{h^0_m} 
\sup_P
\Prob\bigl\{
L( h^0_m)
\geq \epsilon
\bigr\},
\]
and minimax expected risk
\[
\mathcal{M}^{\mathrm{SL}}_{m}(\Hyp)
:=
\inf_{h^0_m} 
\sup_P
\E\left[
L( h^0_m)\right].
\]
Similarly, for any semi-supervised learning algorithm $h_m$, we define its
minimax probability of error
\[
\mathcal{M}^{\mathrm{SSL}}_{\epsilon,m}(\Hyp)
:=
\inf_{h_m} 
\sup_P
\Prob\bigl\{
L( h_m)
\geq \epsilon
\bigr\},
\]
and minimax expected risk
\[
\mathcal{M}^{\mathrm{SSL}}_{m}(\Hyp)
:=
\inf_{h_m} 
\sup_P
\E\left[
L( h_m)\right].
\]
In the previous four equations, the $m$ labeled examples forming $\Z_m$ are
sampled iid from $P(X,Y)$, and the $u$ unlabeled examples forming $\X_u$ are
sampled iid from $P(X)$. Then, the following holds. 

\begin{theorem}
\label{thm:relation-II-SL}
Under the previous definitions, it holds that
\begin{equation}
\label{eq:relations-expect}
\mathcal{M}^{\mathrm{II}}_{N,m}(\Hyp)
\leq
\mathcal{M}^{\mathrm{SSL}}_{m}(\Hyp)
\leq
\mathcal{M}^{\mathrm{SL}}_{m}(\Hyp)
\end{equation}
and
\begin{equation}
\label{eq:relations-prob}
\mathcal{M}^{\mathrm{II}}_{2\epsilon,N,m}(\Hyp) - e^{-2u\epsilon^2}
\leq
\mathcal{M}^{\mathrm{SSL}}_{\epsilon,m}(\Hyp)
\leq
\mathcal{M}^{\mathrm{SL}}_{\epsilon,m}(\Hyp).
\end{equation}
\end{theorem}
\begin{proof}
See Appendix \ref{proof:relation-II-SL}.
\end{proof}

Theorem~\ref{thm:relation-II-SL} shows that the minimax lower bounds for TLSII
lead to minimax lower bounds for both supervised and semi-supervised
learning\footnote{Surprisingly, the TLSII lower bound of
Theorem~\ref{thm:lower-exp-setting2} matches the best known lower bound for
supervised learning \citep[Theorem 14.1]{DGL96}.  Worst case lower bounds for
agnostic semi-supervised learning firstly appeared in \cite[Corollary
4.10]{L09}.  }.
Therefore, the lower bounds from Theorems~\ref{thm:minimax_hp-II}
and~\ref{thm:lower-exp-setting2} imply that the expected risk $L(h_m)$ of any
inductive semi-supervised learning algorithm $h_m$ can not decrease faster than
$\Omega(d/n)$.
Moreover, the algorithm of \citep{H15}, denoted by $\tilde{h}_m$, exhibits
upper bounds of the order $O(d/m)$. Since we can view $\tilde{h}_m$ as a
semi-supervised algorithm which ignores the unlabeled set we can conclude, in
the presence of sufficiently large unlabeled sets, that $\tilde{h}_m$ is a
minimax optimal algorithm for realizable semi-supervised learning.

In short, there always exist distributions such that any semi-supervised
learning algorithm will exhibit no advantage over some supervised learning
algorithm which ignores the unlabeled set.  Said differently, if one makes no
assumptions between the marginal distribution $P(X)$ and the labeling mechanism
$P(Y|X)$ generating the data under study, semi-supervised learning is an
impossible endeavour. This discussion relates to the conjectures of
\citep{BD08} and \citep{scholkopf12anti}. First, \citep{BD08} conjectures that
semi-supervised learning is impossible for any marginal distribution $P(X)$,
since it is always possible to find a bad labeling mechanism $P(Y|X)$ which
renders the unlabeled set useless. Second, \citep{scholkopf12anti} conjectures that
semi-supervised learning is impossible for any marginal distribution $P(X)$ and labeling
mechanism $P(Y|X)$, as long as these two probability distributions share no information.
While our results do not resolve any of these two conjectures, we expect
to add to the discussion about the role of unlabeled data in machine learning.

\section{Conclusion}\label{sec:conclusion}

We provided the first known minimax lower bounds for transductive, realizable,
binary classification, as well as sharp upper bounds for TLSII. For a summary of contributions, see
Table~\ref{table:results}.
In particular, our lower bounds show any transductive learning algorithm 
needs at least
$\Omega\left(\frac{d + \log (1/\delta)}{\epsilon}\right)$
labeled samples to $\epsilon$-learn a hypothesis class $\Hyp$ of VC-dimension
$d < \infty$ with confidence $1-\delta$ when $m\leq u$.
Such lower bound uncovers three important consequences.
First, transductive learning is in general as hard as inductive learning, since
the minimax values of these two problems are $\Theta(d/m)$ (up to logarithmic factor
for TLSI).
Second, unlabeled data does not help general transductive classification, since
supervised learning algorithms, such as ERM and the algorithm of \citet{H15},
match our transductive lower bounds while ignoring the unlabeled set. Third,
our lower bounds for TLSII lead to lower bounds for semi-supervised learning.

We conclude by posing two questions for future research. First, how could we
extend the presented results to agnostic (non-realizable) learning scenarios?
Second, can we improve the $\log N$ factor from the upper bound in TLSI to a
$\log m$ factor?

\acks{IT thanks Ruth Urner for helpful discussions.}

\bibliography{tminimax}

\clearpage
\newpage
\appendix

\section{Proofs of upper bounds for TLSI}
\label{sect:wrong}

Here we discuss the proof of \citep[Corollary 1]{CM06} and provide two slight
improvements. The original result, when adapted to realizable classification,
reads as follows.

\begin{theorem}[Original version]
\label{thm:wrong}
Let $\Hyp$ be a set of classifiers with VC-dimension $d<\infty$.
Let $\hat{h}_m$ be the empirical risk minimizer.
Then, with probability at least $1-\delta$,
\[
L_u(\hat{h}_m)
\leq
\frac{d\log\frac{(m+u)e}{d} + \log\frac{1}{\delta}}{m}.
\]
\end{theorem}

First, inspecting the step from Equation (46) to Equation
(47) of the proof of \citep[Corollary 1]{CM06} reveals the inequality
\begin{align}
-\frac{1}{2}\frac{mu}{m+u}\frac{m+u+2}{m+u-u\epsilon + 1}\frac{u\epsilon}{u\epsilon + 1}\epsilon^2
\label{eq:wrong-inequalities}
\leq
-\frac{1}{2}\frac{mu}{m+u}\epsilon^2.
\end{align}
This inequality is in general
false, and true only if
\[
\frac{m+u+2}{m+u-u\epsilon + 1}\frac{u\epsilon}{u\epsilon + 1} \geq 1,
\]
which is equivalent to
\[
(u\epsilon + 1)^2 \geq m + u + 2,
\]
and
\[
\epsilon \geq \frac{\sqrt{m+u+2} - 1}{u}.
\]
Assume that $u \geq 4$. Then, $1 \leq \sqrt{u}/2$ and 
\[
\frac{\sqrt{m+u+2} - 1}{u}
\geq
\frac{\sqrt{u} - 1}{u}
\geq
\frac{\sqrt{u} }{2u}
=
\frac{1}{2\sqrt{u}}.
\]
In short, Equation~\eqref{eq:wrong-inequalities}, and consequently Theorem~\ref{thm:wrong}, only holds when
\[
\epsilon \geq \frac{1}{2\sqrt{u}}.
\]
This shows that the upper bound of Theorem \ref{thm:wrong} should be replaced with 
\[
\max\left\{
\frac{d\log\frac{(m+u)e}{d} + \log\frac{1}{\delta}}{m},
\frac{1}{2\sqrt{u}}
\right\}.
\]

Second, the upper bound in Theorem~\ref{thm:wrong} has the form $d\log(u + m)/m$.
However, as argued in \citep[Section 2.1.2]{P08}, all upper bounds
in realizable transductive classification should have the form $1/\min(u,m)$. The
discrepancy may be due to an inaccuracy in the proof of \citep[Proposition
1]{CM06}.  Namely, the proof uses the inequality $m \leq u$ but claims,
in between Equations 37 and 38, that ``the case $m \geq u$ can be treated
similarly''.  We conjecture that this is not the case: we could not find any
similar argument that would lead to a result for the $m \geq u$ case.

\subsection{Proof of Theorem~\ref{thm:erm-upper-setting1}}\label{proof:erm-upper-setting1}

Accounting for the previous two remarks, we correct Theorem~\ref{thm:wrong} as
Theorem~\ref{thm:correct-weak2}. First part of Theorem~\ref{thm:erm-upper-setting1}
is a direct consequence of Theorem~\ref{thm:correct-weak2}.

\begin{theorem}[New version of Theorem \ref{thm:wrong}]
\label{thm:correct-weak2}
Let $\Hyp$ be a set of classifiers with VC-dimension $d<\infty$ and assume $u \geq 4$, $m\leq u$.
Then with probability at least $1-\delta$ for the empirical risk minimizer $\hat{h}_m$:
\[
L_u(\hat{h}_m)
\leq
\max\left\{
2\frac{d\log\frac{(m+u)e}{d} + \log\frac{1}{\delta}}{m},
\frac{\sqrt{2}}{u}
\right\}.
\]
\end{theorem}
\begin{proof}
To improve the proof of \citep[Corollary 1]{CM06}, assume that the inequality 
\begin{equation}
\label{eq:wrong-condition}
\frac{m+u+2}{m+u-u\epsilon + 1}\frac{u\epsilon}{u\epsilon + 1} \geq C
\end{equation}
holds for some constant $C>0$. Then, this is equivalent to
\[
-\frac{1}{2}\frac{mu}{m+u}\frac{m+u+2}{m+u-u\epsilon + 1}\frac{u\epsilon}{u\epsilon + 1}\epsilon^2
\leq
-\frac{C}{2}\frac{mu}{m+u}\epsilon^2,
\]
which directly leads to the upper bound of Theorem~\ref{thm:wrong}, with a
multiplicative factor of $C$ in its denominator.  The condition
\eqref{eq:wrong-condition} is equivalent to 
\[
u\epsilon \geq
\frac{\sqrt{(N - NC + 2)^2 + 4NC^2 + 4C^2} - (N - NC + 2)}{2C}.
\]
Let us bound the previous inequality in two different cases:
\begin{itemize}
  \item if $C \geq 1$, then 
  \[
  \frac{\sqrt{(N - NC + 2)^2 + 4NC^2 + 4C^2} - (N - NC + 2)}{2C}
  \geq
  \sqrt{u + m} \geq \sqrt{u}
  \]
  and as a consequence we necessarily have $\epsilon \geq 1/\sqrt{u}$. This condition won't allow us to get an upper bound better than $1/\sqrt{u}$, so we won't consider this choice of $C$.
  \item Second, if $C<1$. Then,
  \begin{align*}
  &\frac{\sqrt{(N - NC + 2)^2 + 4NC^2 + 4C^2} - (N - NC + 2)}{2C}\\
  &\leq
  \frac{\sqrt{4NC^2 + 4C^2}}{2C}
  =
  \sqrt{C^2 + 1}.
  \end{align*}
  This shows that if $u \epsilon \geq \sqrt{2}$ then
  \[
  u \epsilon
  \geq
  \sqrt{C^2 + 1}
  \geq
  \frac{\sqrt{(N - NC + 2)^2 + 4NC^2 + 4C^2} - (N - NC + 2)}{2C}
  \]
  for any $C\in (0,1)$. Therefore, in this second case \eqref{eq:wrong-condition} is always satisfied.
\end{itemize}
Accordingly, we take $C = 1/2$ and obtain the following upper bound:
\[
L_u(\hat{h}_m)
\leq
\max\left\{
2\frac{d\log\frac{(m+u)e}{d} + \log\frac{1}{\delta}}{m},
\frac{\sqrt{2}}{u}
\right\}.
\]

Next, we incorporate three conditions that hold true for our setting. 
These are $d\geq 2$, $m \leq u$, and $m \geq d -
1$. Thus, $m + u \geq d$.  Since $d\mapsto d\log\bigl((m+u)e/d\bigr)$ increases
on $[0, m+u]$, then 
\[
d\log\frac{(m+u)e}{d} 
\geq 
2 \log\frac{(m+u)e}{2}
\geq
2 \log e = 2,
\]
where we used $d\geq 2$ and $u\geq m \geq d-1 \geq 1$.
This shows that
\[
2\frac{d\log\frac{(m+u)e}{d} + \log\frac{1}{\delta}}{m}
\geq
2\frac{2 + \log\frac{1}{\delta}}{m}
\geq
\frac{4}{m}
\geq
\frac{4}{u}
\geq
\frac{\sqrt{2}}{u},
\]
where we used $\delta < 1$.
\end{proof}

Next we prove the second part of Theorem \ref{thm:erm-upper-setting1} by integrating the previous upper bound.
\begin{proof}
First, any non-negative random variable $Z$ with finite expectation satisfies
\[
\E[Z] = \int_{0}^\infty \Prob\{Z > \epsilon\} d\epsilon.
\]
Second, rewrite the first statement of Theorem \ref{thm:erm-upper-setting1} as:
\[
\Prob\left\{
\err(\hat{h}_m, \Z_u^{\pi}) > \epsilon
\right\}
\leq
\min\left\{ \left(\frac{Ne}{d}\right)^d e^{-\epsilon m / 2}, 1\right\},
\]
where we used the fact that probabilities are upper bounded by $1$.
Third, simple computations show that the upper bound of Theorem \ref{thm:erm-upper-setting1} exceeds 1 for
\[
\epsilon \leq \frac{2d\log (Ne/d)}{m} := A.
\]
Combining these three facts, it follows that 
\begin{align*}
\E\left[\err(\hat{h}_m, \Z_u^{\pi})\right]
&=
\int_0^\infty
\Prob\left\{
\err(\hat{h}_m, \Z_u^{\pi}) > \epsilon
\right\}
d\epsilon\\
&\leq
\frac{2d\log (Ne/d)}{m}
+
\int_{A}^\infty \left(\frac{Ne}{d}\right)^d e^{-\epsilon m / 2} d\epsilon\\
&=
\frac{2d\log (Ne/d)+2}{m}.
\end{align*}
\end{proof}

\section{Proofs of lower bounds for TLSI}\label{sec:proofs-lower-i}
Throughout this section, we sample the labeled training set $\Z_m$ and the
unlabeled test set $\Z_u$ as follows. Sample a random permutation $\pi$
distributed uniformly on the symmetric group of $\{1,\dots,N\}$, denoted by
$\Sigma_N$, take $\Z_u:=\{(X_{\pi_i}, Y_{\pi_i})\}_{i=1}^u$, and $\Z_m := \Z_N
\setminus \Z_u$.
We denote the application of the random permutation $\pi$ to the data $(\Z_m,
\Z_u)$ as $(\Z_m^{\pi}, \Z_u^{\pi})$.
\subsection{Proof of Theorem~\ref{thm:minimax_hp}}
\label{proof:minimax_hp}

Under the realizability assumption, if $\Z_N$ contains two pairs $(x_1,y_1)$
and $(x_2,y_2)$ with $x_1=x_2$, this implies that $y_1=y_2$.
We will construct a class of $\Z_N$ in the following way.  Let $x_1,\dots,x_d$
be any distinct points shattered by $\Hyp$, and let $b:=(b_1,\dots,b_d)$ be any
binary string.
We will generate $\Z_N$ by taking $i_j \geq 0$ copies of every pair $(x_j, b_j)$ for $j=1,\dots,d$, where $i_j$ are nonnegative integers such that $\sum_{j=1}^{d} i_j = N$.
We also introduce an order between the elements of $\Z_N$, by first enumerating
the $i_1$ copies of $(x_1, b_1)$, then the $i_2$ copies of $(x_2,b_2)$, and so
on. Therefore, technically speaking, the elements $\Z_N$, $\Z_m$, and $\Z_u$
are ordered multisets.

\subsubsection{Using the probabilistic method to introduce Bernoulli random variables}

Let $k_j(\pi)$ denote the number of copies (multiplicity) of the input $x_j$ contained in
$\Z_u^{\pi}:= \{(X_{\pi_i}, Y_{\pi_i})\}_{i=1}^u$.
Clearly, $\sum_{j=1}^d k_j(\pi) = u$ for any $\pi$.
Because of our design of $\Z_N$, we can write
\begin{align*}
\mathcal{M}_{\epsilon,N,m}^{\mathrm{I}}(\Hyp)
&\geq
\inf_{h_m}\sup_{\{i_j, b_j\}}
\Prob_{\pi} \left\{
\frac{1}{u}\sum_{(x,y)\in \Z_u^{\pi}} \mathbbm{1}\{h_m(x)\neq y\}
-
\inf_{h\in \Hyp}
\frac{1}{u}\sum_{(x,y)\in \Z_u^{\pi}} \mathbbm{1}\{h(x)\neq y\}
\geq \epsilon
\right\}\\
&=
\inf_{h_m}\sup_{\{i_j, b_j\}}
\Prob_{\pi} \left\{
\frac{1}{u}\sum_{j = 1}^d k_j(\pi) \mathbbm{1}\{h_m(x_j)\neq b_j\}
\geq \epsilon
\right\},
\end{align*}
where we used the fact that the best predictor in $\Hyp$
has zero test error, since the inputs in $\Z_N$ are shattered by $\Hyp$.
We continue by introducing a random binary string $B = (B_1, \ldots, B_d)$
distributed uniformly over $\{0,1\}^d$, and lower bounding the supremum over $b$
by the average over $B$:
\begin{align*}
&\inf_{h_m}\sup_{\{i_j, b_j\}}
\Prob_{\pi} \left\{
\frac{1}{u}\sum_{j = 1}^d k_j(\pi) \mathbbm{1}\{h_m(x_j)\neq b_j\}
\geq \epsilon
\right\}\\
&\geq
\inf_{h_m}\sup_{\{i_j\}}
\E_B\left[
\Prob_{\pi} \left\{
\frac{1}{u}\sum_{j = 1}^d k_j(\pi) \mathbbm{1}\{h_m(x_j)\neq b_j\}
\geq \epsilon
\Big| b = B
\right\}
\right]\\
&=
\inf_{h_m}\sup_{\{i_j\}}
\Prob_{\pi,B} \left\{
\frac{1}{u}\sum_{j = 1}^d k_j(\pi) \mathbbm{1}\{h_m(x_j)\neq B_j\}
\geq \epsilon
\right\}.
\end{align*}
Finally, we further lower bound the minimax risk by counting the
missclassifications associated with the points $(x_j, y_j)$ that have all their
copies in the unlabeled set $\Z_u$: 
\begin{align*}
&\inf_{h_m}\sup_{\{i_j\}}
\Prob_{\pi,B} \left\{
\frac{1}{u}\sum_{j = 1}^d k_j(\pi) \mathbbm{1}\{h_m(x_j)\neq B_j\}
\geq \epsilon
\right\}\\
&=
\inf_{h_m}\sup_{\{i_j\}}
\Prob_{\pi,B} \left\{
\frac{1}{u}\sum_{j = 1}^d i_j \mathbbm{1}\{k_j(\pi) = i_j\} \mathbbm{1}\{h_m(x_j)\neq B_j\}\right.
\\
&\left.
\quad\quad\quad\quad\quad+
\frac{1}{u}\sum_{j = 1}^d k_j(\pi) \mathbbm{1}\{k_j(\pi) < i_j\} \mathbbm{1}\{h_m(x_j)\neq B_j\}
\geq \epsilon
\right\}\\
&\geq
\inf_{h_m}\sup_{\{i_j\}}
\Prob_{\pi,B} \left\{
\frac{1}{u}\sum_{j = 1}^d i_j \mathbbm{1}\{k_j(\pi) = i_j\} \mathbbm{1}\{h_m(x_j)\neq B_j\}
\geq \epsilon
\right\}.
\end{align*}

\subsubsection{Setting $i_1=\ldots = i_{d-1}$ to simplify the lower bound}\label{sec:simplifying}

Let $\Delta \in \mathbb{N}_+$ satisfy $\Delta \leq N/(d-1)$.  Under our
assumptions $N \geq 2 m \geq 4(d-1)$, so $N/(d-1) \geq 1$. Thus, the choice of
$\Delta$ is always possible.  We set
\[
(i_1,\dots, i_d) := (\Delta, \dots, \Delta, N - (d-1)\Delta).
\]
For this choice we obviously have $i_j \geq 1$ for $j=1,\dots, d-1$ and $i_d
\geq 0$.  Let us continue the lower bound from the previous section. To this end, ignore 
the copies of $x_d$, and write 
\begin{align}
\notag
&
\inf_{h_m}\sup_{\{i_j\}}
\Prob_{\pi,B} \left\{
\frac{1}{u}\sum_{j = 1}^d i_j \mathbbm{1}\{k_j(\pi) = i_j\} \mathbbm{1}\{h_m(x_j)\neq B_j\}
\geq \epsilon
\right\}\\
\label{eq:particular}
&\geq
\inf_{h_m}\Prob_{\pi,B} \left\{
\sum_{j = 1}^{d-1} \mathbbm{1}\{k_j(\pi) = \Delta\} \mathbbm{1}\{h_m(x_j)\neq B_j\}
\geq \frac{\epsilon u}{\Delta}
\right\}.
\end{align}
By denoting $T(\pi):=\bigl\{j\in\{1,\dots,d-1\}\colon k_j(\pi) = \Delta\bigr\}$, we simplify our notation as
\begin{align*}
\notag
&\inf_{h_m}\Prob_{\pi,B} \left\{
\sum_{j = 1}^{d-1} \mathbbm{1}\{k_j(\pi) = \Delta\} \mathbbm{1}\{h_m(x_j)\neq B_j\}
\geq \frac{\epsilon u}{\Delta}
\right\}\\
&=
\inf_{h_m}\Prob_{\pi,B} \left\{
\sum_{j \in T(\pi)} \mathbbm{1}\{h_m(x_j)\neq B_j\}
\geq \frac{\epsilon u}{\Delta}\right\}.
\end{align*}
Fix any $\pi \in \Sigma_N$.  Note that $x_j$ is not a member of the training
set $\Z_m^{\pi}$, for all $j\in T(\pi)$.  This means that $h_m$ does not depend
on $B_j$, since the learner did not get to see the label $y_j$ during the
training phase.  Because of this reason, when conditioning on $\pi\in\Sigma_N$,
the random variables $h_m(x_j)$ and $B_j$ are independent for $j\in T(\pi)$.
In particular, this implies that the quantities $\mathbbm{1}\{h_m(x_j)\neq
B_j\}$ are $\text{Bernoulli}(\frac{1}{2})$ random variables for all $j \in
T(\pi)$.

Similarly, when conditioning on $\pi\in\Sigma_N$, the random variables
$\mathbbm{1}\{h_m(x_{j'})\neq B_{j'}\}$ and $\mathbbm{1}\{h_m(x_{j''})\neq
B_{j''}\}$ are also independent, for all pairs of different indices $j',j''\in
T(\pi)$.
By denoting 
\begin{align*}
  \eta' &=\mathbbm{1}\{h_m(x_{j'})\neq B_{j'}\},\\
  \eta'' &=\mathbbm{1}\{h_m(x_{j''})\neq B_{j''}\},
\end{align*}
we can verify the independence between $\eta'$ and $\eta''$ as follows:
\begin{align*}
\Prob\left\{
\eta' = 0 \cap
\eta'' = 0 | \pi
\right\}
&=
\sum_{i\in\{0,1\}}
\sum_{j\in\{0,1\}}
\Prob\left\{
h_m(x_{j'}) = i
\cap
B_{j'} = i
 \cap
h_m(x_{j''}) = j
\cap
B_{j''} = j
| \pi
\right\}\\
&=
\frac{1}{4}
\sum_{i\in\{0,1\}}
\sum_{j\in\{0,1\}}
\Prob\left\{
h_m(x_{j'}) = i
 \cap
h_m(x_{j''}) = j
| \pi
\right\}= \frac{1}{4},
\end{align*}
where the second equality follows because the events $E_1:=\{B_{j'}=i\}$,
${E_2:=\{B_{j''}=j\}}$, and $E_3 := \{h_m(x_{j'}) = i \,\cap\, h_m(x_{j''}) = j\}$
are mutually independent given $\pi\in\Sigma_N$, and thus $\Prob\{E_1\cap
E_2\cap E_3\} = P(E_1) P(E_2) P(E_3)$.  The same reasoning applies to all the other
values of $\eta'$ and $\eta''$, which shows that they are
indeed independent.  Summarizing, when conditioning on $\pi\in\Sigma_N$, the quantity
$\sum_{j \in T(\pi)} \mathbbm{1}\{h_m(x_j)\neq B_j\}$ is a Binomial random
variable with parameters $(|T(\pi)|,0.5)$.  Thus, we can write
\begin{align}
\notag
&\inf_{h_m}\Prob_{\pi,B} \left\{
\sum_{j \in T(\pi)} \mathbbm{1}\{h_m(x_j)\neq B_j\}
\geq \frac{\epsilon u}{\Delta}\right\}\\
\notag
&=
\inf_{h_m}\sum_{\pi'\in \Sigma_N}\frac{1}{N!}
\Prob_{\pi,B} \left\{
\sum_{j \in T(\pi)} \mathbbm{1}\{h_m(x_j)\neq B_j\}
\geq \frac{\epsilon u}{\Delta} \bigg| \pi = \pi'\right\}\\
\notag
&=
\sum_{\pi'\in \Sigma_N}\frac{1}{N!}
\Prob_{\pi,B} \left\{
\mathrm{Binom}( |T(\pi')|, 1/2)
\geq \frac{\epsilon u}{\Delta} \right\}
\\
\label{eq:proof-starting-point}
&=
\sum_{M=0}^{d-1}\frac{\bigl|\{ \pi\in\Sigma_N \colon |T(\pi)| = M\}\bigr|}{N!}
\Prob_{B} \left\{
\mathrm{Binom}( M, 1/2)
\geq \frac{\epsilon u}{\Delta} \right\},
\end{align}
where the equalities follow from the law of total probability, replacing sums
of indicator functions with Binomial random variables, and breaking the
symmetric group $\Sigma_N$ in $d$ blocks, each of them containing permutations
$\pi$ with same $|T(\pi)|$. 

Observe that Theorem~\ref{thm:minimax_hp} is composed by two statements. We
now proceed to prove each of them separately.

\subsubsection{Proof of Theorem~\ref{thm:minimax_hp}, Statement \eqref{item:Theorem-1}, $d \geq 7$}\label{sec:minimax_hp_1}

We can further lower bound \eqref{eq:proof-starting-point} as follows:
\begin{align}
\notag
\mathcal{M}_{\epsilon,N,m}^{\mathrm{I}}(\Hyp)
&\geq
\sum_{M=2\left\lceil\frac{\epsilon u}{\Delta}\right\rceil}^{d-1}\frac{\bigl|\{ \pi\in\Sigma_N \colon |T(\pi)| = M\}\bigr|}{N!}
\Prob_{B} \left\{
\mathrm{Binom}( M, 1/2)
\geq \frac{\epsilon u}{\Delta} \right\}\\
\notag
&\geq
\sum_{M=2\left\lceil\frac{\epsilon u}{\Delta}\right\rceil}^{d-1}\frac{\bigl|\{ \pi\in\Sigma_N \colon |T(\pi)| = M\}\bigr|}{N!}
\Prob_{B} \left\{
\mathrm{Binom}\left( 2\left\lceil\frac{\epsilon u}{\Delta}\right\rceil, 1/2\right)
\geq \frac{\epsilon u}{\Delta} \right\}\nonumber\\
&\geq
\sum_{M=2\left\lceil\frac{\epsilon u}{\Delta}\right\rceil}^{d-1}\frac{\bigl|\{ \pi\in\Sigma_N \colon |T(\pi)| = M\}\bigr|}{N!}\cdot \frac{1}{2},
\label{eq:expantionP}
\end{align}
where the inequalities follow by truncating the sum to start at $M=
2\left\lceil\frac{\epsilon u}{\Delta}\right\rceil$, minimizing the number of
trials in the Binomial distributions, and $\Prob(\text{Binom}(2 a, 1/2) \geq a)
\geq 1/2$.

Next, we will count the number of different permutations $\pi$ satisfying
$|T(\pi)| = M$, for each $M\in \{2\lceil\epsilon u/{\Delta}\rceil,\dots,
d-1\}$.
First of all, there are
\[
{d-1 \choose M}
\]
ways to choose $M$ distinct elements
$\{x_{\ell^*_1},\dots, x_{\ell^*_M}\}$ from the set $x_1,\dots,x_{d-1}$, which
will not be contained in the training set.
Also, recall that at the beginning of our proof we defined  the test set $\Z^\pi_u$
to contain the elements with indices $\{\pi_1,\dots,\pi_u\}$.
Therefore, 
we need to guarantee that $\Z_u$ contains $\Delta$
copies of each $\{x_{\ell^*_1},\dots, x_{\ell^*_M}\}$.
This leads to the condition $u \geq \Delta M$, which is satisfied if $u \geq
\Delta (d-1)$, since $M\leq d-1$.
We will guarantee this condition later, by a specific choice of~$\Delta$.
In any case, there are exactly
\[
\Delta! \cdot \left\lbrace
{u \choose \Delta} \cdot
{u - \Delta \choose \Delta}
\cdots
{u - \Delta (M-1) \choose \Delta} \right\rbrace
=
\frac{u!}{(u - \Delta M)!}
\]
ways to place the indices of the $\Delta M$ test points in the first $u$
coordinates of $\pi$.
Now, let us consider the training set. For this, we need to ensure that every
element from $\{x_1,\dots, x_{d-1}\} \setminus \{x_{\ell^*_1},\dots,
x_{\ell^*_M}\}$ appears at least once in the training set.
To this end, choose $(d - 1) - M$ indices out of $\{1,\dots,N\}$, corresponding
to some elements from $\{x_1,\dots, x_{d-1}\} \setminus \{x_{\ell^*_1},\dots,
x_{\ell^*_M}\}$, and distribute them within the last $m$ coordinates of $\pi$ (this is
possible, since $m \geq d-1$).
There are 
\[
{m \choose d - 1 - M}(d - 1 -M)!
\]
ways to do so.
The remaining $N - \Delta M - d + 1 + M$ indices can be distributed among the remaining coordinates of $\pi$ in any of the 
\[
(N - \Delta M - d + 1 + M)!
\]
possible orders.  The previous four equations in display lead
to a \emph{lower bound} on the number of permutations $\pi$ satisfying our
demands (because of the training set part, where we only lower bounded the total number of different permutations). 
Together with the $\frac{1}{N!}$ denominator from \eqref{eq:expantionP},
\begin{align*}
&
{d-1 \choose M}
\frac{u!}{(u - \Delta M)!} \cdot
{m \choose d - 1 - M}(d - 1 -M)! \cdot
(N - \Delta M - d + 1 + M)! \cdot
\frac{1}{N!}\\
&=
{d-1 \choose M}\frac{m!(N - \Delta M - d + 1 + M)!u!(d - 1 -M)!}{(d-1-M)!(m-d+1+M)!N!(u - \Delta M)!}\\
&=
\frac{{d-1 \choose M}}{{N \choose u}}\frac{(N - \Delta M - d + 1 + M)!}{(m-d+1+M)!(u - \Delta M)!}\\
&=
\frac{{d-1 \choose M}}{{N \choose u}}{N - \Delta M - d + 1 + M \choose u - \Delta M}.
\end{align*}
Therefore, continue lower bounding \eqref{eq:expantionP} as
\begin{align}
\mathcal{M}_{\epsilon,N,m}^{\mathrm{I}}(\Hyp)
\notag
&\geq
\frac{1}{2}\sum_{M=2\left\lceil\frac{\epsilon u}{\Delta}\right\rceil}^{d-1}
\frac{{d-1 \choose M}}{{N \choose u}} {N - \Delta M - d + 1 + M \choose u - \Delta M}\\
\label{eq:interm}
&=
\frac{1}{2}\sum_{M=2\left\lceil\frac{\epsilon u}{\Delta}\right\rceil}^{d-1}
\frac{{d-1 \choose M} {N - d + 1 \choose u - M}}{{N \choose u}} \frac{{N - d + 1 - (\Delta-1) M \choose u - M -(\Delta-1) M}}{{N - d + 1 \choose u - M}},
\end{align}
where the equality holds as long as $u \geq M$, $N \geq d - 1$, and $N - d + 1
\geq u - M$. 
These three inequalities are fulfilled because of the assumptions $N \geq u
\geq m \geq 8(d-1)\geq d-1$.  Using $M \leq d-1$ together with the first part
of Lemma~\ref{lemma:binomial-ratio} with $n=u+m-d+1$, $i=(\Delta-1)M$ and
$k=u-M$, we obtain
\begin{align*}
\frac{{N - d + 1 - (\Delta-1) M \choose u - M -(\Delta-1) M}}{{N - d + 1 \choose u - M}}
&\geq
\left(
1 - \frac{(\Delta - 1)M}{u - M + 1}
\right)^{m + M - d + 1}
\geq
\left(
1 - 
\frac{(\Delta - 1)(d-1)}{u - d + 2}\right
)^{m}.
\end{align*}
Plugging the last inequality back to \eqref{eq:interm} yields 
\begin{equation}
\label{eq:proof-hgfactor}
\mathcal{M}_{\epsilon,N,m}^{\mathrm{I}}(\Hyp)
\geq
\frac{1}{2}
\left(
1 - \frac{(\Delta - 1)(d-1)}{u - d + 2}
\right)^{m}
\sum_{M=2\left\lceil\frac{\epsilon u}{\Delta}\right\rceil}^{d-1}
\frac{{d-1 \choose M} {N - d + 1 \choose u - M}}{{N \choose u}}.
\end{equation}

The next step is to realize that the summands in \eqref{eq:proof-hgfactor} are
\emph{hypergeometric random variables}.  Namely, a random variable $Z$ taking
values in $\{0,1,\dots,d-1\}$ is called \emph{hypergeometric}, with
parameters~$(N, d-1,u)$, if
\[
\Prob\{Z = k\}
=
\frac{{d-1 \choose k} {N - d + 1 \choose u - k}}{{N \choose u}},\quad k = 0,\dots, d-1.
\]
Relevant to our interests, the expressions for a mean and a variance of a hypergeometric random variable $Z$ with parameters $(N,d-1,u)$ are
\[
\E[Z] = u\frac{d-1}{N},
\quad
\mathrm{Var}[Z] = u\frac{(d-1)(N-d+1)m}{N^2(N-1)}.
\]
We may now use these expressions, together with
$\mathrm{Var}(-Z)=\mathrm{Var}(Z)$, and the Chebyshev-Cantelli inequality
\cite[Theorem A.17]{DGL96}, to obtain
\begin{align}
\notag
\sum_{M=2\left\lceil\frac{\epsilon u}{\Delta}\right\rceil}^{d-1}
\frac{{d-1 \choose M} {N - d + 1 \choose u - M}}{{N \choose u}}
&= 
\Prob\left\{Z \geq 2\left\lceil\frac{\epsilon u}{\Delta}\right\rceil\right\}\\
&=
1 - \Prob\left\{-Z  - \E[-Z] > \E[Z] -  2\left\lceil\frac{\epsilon u}{\Delta}\right\rceil  \right\}\nonumber\\
&\geq
1 - \frac{\mathrm{Var}[Z]}{\mathrm{Var}[Z] + \left( \E[Z] -  2\left\lceil\frac{\epsilon u}{\Delta}\right\rceil\right)^2},
\label{eq:CC-hg}
\end{align}
which holds as long as 
\[
{\E[Z] = u\frac{d-1}{N} \geq 2\left\lceil\frac{\epsilon u}{\Delta}\right\rceil}.
\]
We satisfy this condition by setting $\Delta = \lceil \frac{7 N\epsilon}{d-1} \rceil \geq 1$.
In addition, $d \geq 7$ and $u \geq N/2$, so 
\begin{align}
\notag
2\left\lceil\frac{\epsilon u}{\Delta}\right\rceil
&\leq
2\left(\frac{ u(d-1)}{7 N} + 1\right)
=
\frac{u(d-1)}{N}\left(\frac{2}{7 } + \frac{2N}{u(d-1)}\right)\\
\label{eq:proof-ch-c-1}
&\leq
\frac{u(d-1)}{N}\left(\frac{2}{7 } + \frac{2}{3}\right) = \frac{20}{21}\E[Z].
\end{align}
Next, we show that all the conditions that we have required so far are satisfied for our choice of~$\Delta$. 
To this end, we need to verify that
$\Delta \leq N/(d-1)$ and 
$u \geq \Delta(d-1)$.
The first condition follows from the second one.
To check the second condition, we notice that $8(d-1)\leq m \leq u$ and thus $(d-1)/u\leq 1/8$, which leads to 
\[
\Delta
\leq
1 + \frac{7N \epsilon}{d-1}
=
\frac{u}{d-1}\left(\frac{d-1}{u} + \frac{7N \epsilon}{u}\right)
\leq
\frac{u}{d-1}\left(\frac{1}{8} + \frac{14}{32}\right)
\leq
\frac{u}{d-1},
\]
where we have used $\epsilon \leq 1/32$ and $u \geq N/2$. 

Using the expressions for the mean and variance of hypergeometric random
variables, together with \eqref{eq:CC-hg} and \eqref{eq:proof-ch-c-1}, it
follows that 
\begin{align*}
\sum_{M=2\left\lceil\frac{\epsilon u}{\Delta}\right\rceil}^{d-1}
\frac{{d-1 \choose M} {N - d + 1 \choose u - M}}{{N \choose u}}
&\geq
1 - \frac{\mathrm{Var}[Z]}{\mathrm{Var}[Z] + \left( \E[Z] -  2\left\lceil\frac{\epsilon u}{\Delta}\right\rceil\right)^2}\\
&=
1 - \frac{u\frac{(d-1)(N-d+1)m}{N^2(N-1)}}{u\frac{(d-1)(N-d+1)m}{N^2(N-1)} + \left( u\frac{d-1}{N} -  2\left\lceil\frac{\epsilon u}{\Delta}\right\rceil\right)^2}\\
&\geq
1 - \frac{u\frac{(d-1)(N-d+1)m}{N^2(N-1)}}{u\frac{(d-1)(N-d+1)m}{N^2(N-1)} + \frac{1}{21^2}\E^2[Z]}.
\end{align*}
Moreover,
\begin{align}
\sum_{M=2\left\lceil\frac{\epsilon u}{\Delta}\right\rceil}^{d-1}
\frac{{d-1 \choose M} {N - d + 1 \choose u - M}}{{N \choose u}}\nonumber
&\geq
1 - \frac{u\frac{(d-1)(N-d+1)m}{N^2(N-1)}}{u\frac{(d-1)(N-d+1)m}{N^2(N-1)} + \frac{1}{21^2}\frac{u^2(d-1)^2}{N^2}}\nonumber\\
&=
1 - \frac{(N-d+1)m}{(N-d+1)m + \frac{1}{21^2}u(d-1)(N-1)}\nonumber\\
&\geq
1 - \frac{N-d+1}{N-d+1 + \frac{1}{21^2}(d-1)(N-1)}\nonumber\\
&\geq
1 - \frac{N-6}{N-6 + \frac{6}{21^2}(N-1)}\nonumber\\
&=
1 - \frac{1}{1 + \frac{6}{21^2}\frac{N-1}{N-6}}\nonumber\\
&\geq
1 - \frac{1}{1 + \frac{6}{21^2}}
=
\frac{6}{21^2 + 6}
=
\frac{6}{447} > \frac{1}{75},\label{eq:piece1}
\end{align}
where we used $u \geq m$, $d \geq 7$, 
and the fact that $x \mapsto \frac{x}{x + c}$ monotonically increases for $c>0$.

Also since $N \geq 2m \geq 16(d-1)$ and $\epsilon \leq 1/32$ we have
\begin{align*}
\frac{(\Delta - 1)(d-1)}{u - d + 2}
&
\leq
\frac{7N\epsilon}{N/2 - d + 1}
\leq
\frac{7N\epsilon}{N/2 - N/16} = 16\epsilon \leq \frac{1}{2} < 1.
\end{align*}
Using $1 - x \geq e^{-x/(1-x)}$, which holds for $x\in[0,1)$, and $\epsilon \leq 1/32$ we conclude that
\begin{equation}
\left(
1 - 
\frac{(\Delta - 1)(d-1)}{u - d + 2}\right
)^{m}
\geq
\exp\left(
-\frac{112}{7}
\frac{\epsilon m}{1 -\frac{112}{7} \epsilon }
\right)
\geq
e^{-32 \epsilon m}.\label{eq:piece2}
\end{equation}
Plugging \eqref{eq:piece1} and \eqref{eq:piece2} into \eqref{eq:proof-hgfactor}
we finally lower-bound the minimax probability as 
\[
\mathcal{M}_{\epsilon,N,m}^{\mathrm{I}}(\Hyp)
\geq
\frac{1}{150}
e^{-32 m\epsilon }.
\]

\subsubsection{Proof of Theorem~\ref{thm:minimax_hp}, Statement \eqref{item:Theorem-1}, $d < 7$}

Let $\Delta = \lceil \frac{7 N\epsilon}{d-1} \rceil \geq 1$. Then,
\[
\frac{\epsilon u}{\Delta}
\leq
\frac{u(d-1)}{7N}
<
1,
\]
Using this inequality in \eqref{eq:proof-starting-point}, we have
\begin{align}
\notag
\mathcal{M}_{\epsilon,N,m}^{\mathrm{I}}(\Hyp)
&\geq
\sum_{M=0}^{d-1}\frac{\bigl|\{ \pi\in\Sigma_N \colon |T(\pi)| = M\}\bigr|}{N!}
\Prob_{B} \left\{
\mathrm{Binom}( M, 1/2)
\geq \frac{\epsilon u}{\Delta} \right\}\\
\notag
&=
\sum_{M=1}^{d-1}\frac{\bigl|\{ \pi\in\Sigma_N \colon |T(\pi)| = M\}\bigr|}{N!}
\Prob_{B} \left\{
\mathrm{Binom}( M, 1/2)
\geq 1 \right\}\\
\notag
&=
\sum_{M=1}^{d-1}\frac{\bigl|\{ \pi\in\Sigma_N \colon |T(\pi)| = M\}\bigr|}{N!}
( 1 - 2^{-M})\\
\label{eq:proof-dleq-lower}
&\geq
\frac{1}{2}
\sum_{M=1}^{d-1}\frac{\bigl|\{ \pi\in\Sigma_N \colon |T(\pi)| = M\}\bigr|}{N!}.
\end{align}
Reusing the computations from Section~\ref{sec:minimax_hp_1}, we obtain the bound
\[
\sum_{M=1}^{d-1}\frac{\bigl|\{ \pi\in\Sigma_N \colon |T(\pi)| = M\}\bigr|}{N!}
\geq
\left(
1 - 
\frac{(\Delta - 1)(d-1)}{u - d + 2}\right
)^{m}
\sum_{M=1}^{d-1}
\frac{{d-1 \choose M} {N - d + 1 \choose u - M}}{{N \choose u}}.
\]
Notice that the previous sum runs over all the support of the hypergeometric
distribution, except for $M=0$. Thus, 
\begin{equation}
\sum_{M=1}^{d-1}
\frac{{d-1 \choose M} {N - d + 1 \choose u - M}}{{N \choose u}}
=
1 - \frac{{N - d + 1 \choose u}}{{N \choose u}}.\label{eq:hgd_again}
\end{equation}
To analyze this term, note that
\begin{align*}
\frac{{N - d + 1 \choose u}}{{N \choose u}}&=
\frac{(N-d + 1)!m!}{(m-d + 1)!N!}\\
&=
\frac{(m - d+2) \cdots m}{(N-d + 2)\cdots N}\\
&=
\left(1 - \frac{u}{N-d+2}\right) \cdots \left(1 - \frac{u}{N}\right)\\
&\leq
 \left(1 - \frac{u}{N}\right)^{d-1} \leq \frac{1}{2},
\end{align*}
where the second equality is due to, and the last inequality is due to $u \geq
\frac{1}{2} N$ and $d \leq 2$. Plugging this constant into
\eqref{eq:hgd_again}, we obtain
\[
\sum_{M=1}^{d-1}
\frac{{d-1 \choose M} {N - d + 1 \choose u - M}}{{N \choose u}},
\geq 
\frac{1}{2},
\]
which together with \eqref{eq:proof-dleq-lower} gives
\[
\mathcal{M}_{\epsilon,N,m}^{\mathrm{I}}(\Hyp)
\geq
\frac{1}{4}
\left(
1 - 
\frac{(\Delta - 1)(d-1)}{u - d + 2}
\right)^{m}.
\]
Using again \eqref{eq:piece2}, if follows that 
\[
\left(
1 - 
\frac{(\Delta - 1)(d-1)}{u - d + 2}\right
)^{m}
\geq
e^{-32 \epsilon m},
\]
which leads to the following lower bound for our minimax probability:
\[
\mathcal{M}_{\epsilon,N,m}^{\mathrm{I}}(\Hyp)
\geq
\frac{1}{4}
e^{-32m\epsilon}.
\]

\subsubsection{Proof of Theorem~\ref{thm:minimax_hp}, Statement \eqref{item:Theorem-2}, $\frac{\epsilon u}{\lfloor N / m \rfloor} \geq 1$}

Start with \eqref{eq:proof-starting-point}, and lower bound as 
\begin{align}
\notag
\mathcal{M}_{\epsilon,N,m}^{\mathrm{I}}(\Hyp)
&\geq
\sum_{M=2\left\lfloor \frac{\epsilon u}{\Delta}\right\rfloor}^{d-1}
\Prob\{ |T(\pi)| = M \}
\cdot \Prob_{B} \left\{
\mathrm{Binom}( M, 1/2)
\geq \frac{\epsilon u}{\Delta} \right\}\nonumber\\
&\geq
\Prob\left\{ |T(\pi)| \geq 2\left\lfloor \frac{\epsilon u}{\Delta}\right\rfloor \right\} \cdot
\Prob \left\{
\mathrm{Binom}\left( 2\left\lfloor \frac{\epsilon u}{\Delta}\right\rfloor, \frac12\right)
\geq \frac{\epsilon u}{\Delta} \right\},
\label{eq:proof-21-general}
\end{align}
where the last inequality follows by considering only the first summand. To
lower bound the second factor of \eqref{eq:proof-21-general}, set $\Delta =
\lfloor N / m \rfloor \geq 2$. This choice of $\Delta$ satisfies our conditions
$u \geq (d-1)\Delta$ and $N \geq (d-1)\Delta$, since
\[
(d - 1)\Delta
\leq
\frac{(d-1)N}{m}
\leq
\frac{N}{2}
\leq
\frac{2u}{2} = u,
\]
where we have used $u \geq m \geq 2(d-1)$ and $u \geq N/2$. Next, note that 
\[
 \frac{\epsilon u}{\Delta}
\geq
1.
\]
Using this inequality and \cite[Lemma A.3]{DGL96}, write 
\begin{align}
\notag
\Prob \left\{
\mathrm{Binom}\left( 2\left\lfloor \frac{\epsilon u}{\Delta}\right\rfloor, \frac12\right)
\geq \frac{\epsilon u}{\Delta} \right\}
&\geq
\frac12 - \frac12
\Prob \left\{
\mathrm{Binom}\left( 2\left\lfloor \frac{\epsilon u}{\Delta}\right\rfloor, \frac12\right)
= \left\lfloor \frac{\epsilon u}{\Delta}\right\rfloor \right\}\\
\label{eq:proof-21-binom}
&\geq
\frac{1}{2}\left(1 - \sqrt{\frac{1}{4\pi\left\lfloor \frac{\epsilon u}{\Delta}\right\rfloor}}\right)
\geq
\frac{1}{2}\left(1 - \sqrt{\frac{1}{4\pi}}\right)
> \frac{1}{3},
\end{align}
where the first inequality is due the structure of a Binomial distribution with
an even number of trials. 

To lower bound the first factor of \eqref{eq:proof-21-general}, observe that
\[
2\left\lfloor \frac{\epsilon u}{\Delta}\right\rfloor
\leq
\frac{2\epsilon u}{\Delta}
\leq
\frac{2\epsilon u}{\frac{N}{m}-1}
=
\frac{2\epsilon um}{u}
=
2\epsilon m
\leq
\frac{d-1}{12}.
\]
Using the previous inequality, it follows that
\begin{align}
\label{eq:proof-21-Tlower}
\Prob\left\{ |T(\pi)| \geq 2\left\lfloor \frac{\epsilon u}{\Delta}\right\rfloor \right\}
\geq
\Prob\left\{ |T(\pi)| \geq \frac{d-1}{12} \right\}.
\end{align}
We will lower bound the previous probability by exploiting the fact that
$k_i(\pi)$ follows a hypergeometric distribution with parameters $(N, \Delta,
u)$, for all $i\in\{1,\dots, d-1\}$. First, obtain the expectation
\[
\E |T(\pi)|
=
\E \left[
\sum_{i=1}^{d-1} \mathbbm{1}\{ k_i(\pi) = \Delta \}
\right]
=
(d-1) \frac{{N - \Delta \choose u - \Delta}}{{N \choose u}},
\]
which can be further lower bounded as
\begin{align}
\notag
\frac{{N - \Delta \choose u - \Delta}}{{N \choose u}} 
&\stackrel{i)}{\geq}
\left(1 - \frac{m}{N - \lfloor N/m\rfloor + 1}\right)^{\lfloor N/m \rfloor}
\geq
\left(1 - \frac{m}{N - N/m + 1}\right)^{N/m}\\
\notag
&=
\left(1 - \frac{m^2}{N(m-1) + m}\right)^{N/m}
\geq
\left(1 - \frac{m^2}{(N+1)(m-1)}\right)^{N/m}\\
\notag
&\stackrel{i)}{\geq}
\left(1 - \frac{9}{8}\frac{m}{N+1}\right)^{N/m}
=
\left(1 - \frac{9}{8}\frac{m}{N+1}\right)^{\left(\frac{8(N+1)}{9m}-1\right)\frac{9N}{8(N+1)}}\left(1 - \frac{9}{8}\frac{m}{N+1}\right)^{\frac{9N}{8(N+1)}}\\
\label{eq:proof-d-lower-bound}
&\stackrel{iii)}{\geq}
e^{-\frac{9N}{8(N+1)}}\left(1 - \frac{9}{8}\frac{m}{N+1}\right)^{\frac{9N}{8(N+1)}}
\stackrel{iv)}{\geq}
e^{-\frac{9}{8}}\left(1 - \frac{9}{8}\frac{m}{N+1}\right)^{\frac{9}{8}}
\stackrel{v)}{\geq}
\left(\frac{1}{e} \cdot \frac{7}{16}\right)^{\frac{9}{8}} > \frac{1}{8},
\end{align}
where the previous follows because i) Lemma~\ref{lemma:binomial-ratio}, ii) $m
\geq 9$, iii) $8(N+1)\geq 8(2m + 1) \geq 16m > 9m$ and $(1 - 1/x)^{x-1}$
monotonically decreases to $e^{-1}$ for $x \geq 1$, iv) $N/(N+1) \leq 1$ for
positive $N$, and v) $m/(N+1) < 1/2$.

Second, obtain the variance $\mathrm{Var}[|T(\pi)|] \leq (d-1)^2/4$, since
$|T(\pi)| \leq d-1$.  Using the obtained expectation and variance, together
with the Chebyshev-Cantelli inequality, 
\begin{align*}
\Prob\left\{ |T(\pi)| \geq \frac{d-1}{12} \right\}
&=
1 - \Prob\left\{ (-|T(\pi)|) - \E[-|T(\pi)|] > \E[|T(\pi)|] - \frac{d-1}{12} \right\}\\
&\geq
1 - \Prob\left\{ (-|T(\pi)|) - \E[-|T(\pi)|] > \frac{d-1}{8} - \frac{d-1}{12} \right\}\\
&\geq
1 - \frac{(d-1)^2/4}{(d-1)^2/4 + \left(\frac{d-1}{8} - \frac{d-1}{12}\right)^2}\\
&\geq
1 - \frac{1}{1 + \left(2 - \frac{4}{3}\right)^2} > \frac{3}{10}.
\end{align*}
Plugging together the previous inequality with \eqref{eq:proof-21-general},
\eqref{eq:proof-21-binom}, and \eqref{eq:proof-21-Tlower}, we obtain our result 
\[
\mathcal{M}_{\epsilon,N,m}^{\mathrm{I}}(\Hyp)
\geq
\frac{1}{10}.
\]

\subsubsection{Proof of Theorem~\ref{thm:minimax_hp}, Statement \eqref{item:Theorem-2}, $\frac{\epsilon u}{\lfloor N / m \rfloor} < 1$}

Let $\Delta = \lfloor N/m \rfloor \geq 2$. Then,
\[
\frac{\epsilon u}{\Delta} <1.
\]
Using this inequality in \eqref{eq:proof-starting-point}, we get
\begin{align*}
\mathcal{M}_{\epsilon,N,m}^{\mathrm{I}}(\Hyp)
&\geq
\sum_{M=0}^{d-1}\frac{\bigl|\{ \pi\in\Sigma_N \colon |T(\pi)| = M\}\bigr|}{N!}
\Prob_{B} \left\{
\mathrm{Binom}( M, 1/2)
\geq \frac{\epsilon u}{\Delta} \right\}\\
&=
\sum_{M=1}^{d-1}\frac{\bigl|\{ \pi\in\Sigma_N \colon |T(\pi)| = M\}\bigr|}{N!}
\Prob_{B} \left\{
\mathrm{Binom}( M, 1/2)
\geq 1 \right\}\\
&\geq
\Prob\{|T(\pi)| \geq 1\} \cdot \frac{1}{2}\\
&=
\frac{1}{2}
-
\frac{1}{2}\cdot \Prob\{|T(\pi)| = 0\},
\end{align*}
where the first equality is due to $\Prob(\text{Binom}(0, 1/2) \geq 0) = 1$ and
$\Prob(\text{Binom}(M, 1/2) \geq A) = \Prob(\text{Binom}(M, 1/2) \geq 1)$ for
$A\in(0,1]$, the second is due to $\Prob(\text{Binom}(M, 1/2) \geq 1) \geq 1/2$
for $M \geq 1$, and the last equality is due to the law of total probability. 

Next, observe that
\[
\Prob\{|T(\pi)| = 0\}
=
\Prob\left\{\sum_{j=1}^{d-1} \mathbbm{1}\{k_j(\pi) = \Delta\} = 0\right\}
=
\Prob\left\{\bigcap_{j=1}^{d-1} \{k_j(\pi) < \Delta\}\right\}
\leq
\Prob\{k_1(\pi) < \Delta\}.
\]
The quantity $k_1(\pi)$ follows a hypergeometric distribution with parameters $(N, \Delta, u)$. Therefore, use \eqref{eq:proof-d-lower-bound} to obtain 
\[
\Prob\{k_1(\pi) < \Delta\}
=
1 - \Prob\{k_1(\pi) = \Delta\}
=
1 - \frac{{\Delta \choose \Delta}{N - \Delta \choose u - \Delta}}{{N \choose u}}
<
1 - \frac{1}{8} = \frac{7}{8},
\]
and conclude
\[
\mathcal{M}_{\epsilon,N,m}^{\mathrm{I}}(\Hyp) \geq \frac{1}{2} - \frac{7}{16} = \frac{1}{16}.
\]

\subsection{Proof of Theorem~\ref{thm:minimax_exp}}
\label{section:minimax_exp}
\begin{proof}
We continue to use the notations introduced at the beginning of this Appendix.
Start by choosing a collection of points $x_1,\dots, x_d$ shattered by $\Hyp$,
and introduce the family of sets $\Z_N$ parametrized by the vectors
$(i_1,\dots,i_d)$ and $(b_1,\dots,b_d)$.
Then,
\[
\mathcal{M}_{N,m}^{\mathrm{I}}(\Hyp)
\geq
\inf_{h_m}
\sup_{\{i_j, b_j\}}
\E_{\pi} \left[
\frac{1}{u}\sum_{(x,y)\in \Z_u^{\pi}} \mathbbm{1}\{h_m(x)\neq y\}
-
\inf_{h\in \Hyp}
\frac{1}{u}\sum_{(x,y)\in \Z_u^{\pi}} \mathbbm{1}\{h(x)\neq y\}
\right].
\]
Since $x_1, \ldots, x_d$ are shattered by $\Hyp$, the risk of the best
predictor in $\Hyp$ is equal to zero.  Then, lower bound the supremum over
$\{b_j\}_{j=1}^d$ by the expectation over $B$, distributed uniformly in
$\{0,1\}^d$, and obtain 
\[
\mathcal{M}_{N,m}^{\mathrm{I}}(\Hyp)
\geq
\inf_{h_m}
\sup_{\{i_j\}}
\E_B
\E_{\pi} \left[\frac{1}{u}\sum_{(x,y)\in \Z_u^{\pi}} \mathbbm{1}\{h_m(x)\neq y\}
\right].
\]
The previous expression is equivalent to
\begin{align*}
\mathcal{M}_{N,m}^{\mathrm{I}}(\Hyp)
&\geq
\inf_{h_m}
\sup_{\{i_j\}}
\E_B
\E_{\pi} \left[
\frac{1}{u}
\sum_{j=1}^{d-1} i_j\mathbbm{1}\{k_j(\pi) = i_j\} \mathbbm{1}\{h_m(x_j)\neq B_j\}
\right].
\end{align*}
Fix any $\pi \in \Sigma_N$. If $k_j(\pi) = i_j$ for some $j\in\{1,\dots,d-1\}$,
then $h_m$ does not depend on $B_j$, since the learning algorithm did not see
$B_j$ during the training phase.  Consequently, for such $j$ we have $\E_B
\left[ \mathbbm{1}\{h_m(x_j)\neq B_j\}\right] = 1/2$.  (We used this same
argument in Section~\ref{sec:simplifying}.) Therefore, we conclude that
\[
\mathcal{M}_{N,m}^{\mathrm{I}}(\Hyp)
\geq
\sup_{\{i_j\}}
\frac{1}{2u}\sum_{j = 1}^{d-1} i_j \mathbb{P}\left\{
k_j(\pi) = i_j
\right\}.
\]
As usual, and for every $j\in \{1,\dots,d-1\}$, the quantity $k_j(\pi)$ is a
random variable following a hypergeometric distribution with parameters $(N,
i_j, u)$, so 
\begin{equation*}
\mathbb{P}\left\{
k_j(\pi) = i_j
\right\} 
=
\frac{{i_j \choose i_j}{N - i_j \choose u - i_j}}{{N \choose u}},
\end{equation*}
and
\begin{equation}
\label{eq:eminimax_basic}
\mathcal{M}_{N,m}^{\mathrm{I}}(\Hyp)
\geq
\sup_{\{i_j\}}
\frac{1}{2u}\sum_{j = 1}^{d-1} i_j \frac{{N - i_j \choose m}}{{N \choose m}}.
\end{equation}

We will now consider the assignment
\[
(i_1,i_2\dots,i_d) = \left( \left\lfloor \frac{N}{m} \right\rfloor, \dots, \left\lfloor \frac{N}{m} \right\rfloor, N - \left\lfloor \frac{N}{m} \right\rfloor(d - 1)\right).
\]
Since $m \geq d - 1$ we obviously have $i_j \geq 1$ for $j=1,\dots, d-1$ and $i_d \geq 0$.
Therefore, using this choice, Equation~\ref{eq:eminimax_basic} can be rewritten as
\begin{align}
\mathcal{M}_{N,m}^{\mathrm{I}}(\Hyp)\nonumber
&\geq
\frac{d - 1}{2u}\left\lfloor \frac{N}{m} \right\rfloor \frac{{N - \lfloor N/m \rfloor \choose m}}{{N \choose m}}\nonumber\\
&\geq
\frac{d - 1}{2(N-m)}\left( \frac{N}{m} - 1\right) \frac{{N - \lfloor N/m \rfloor \choose m}}{{N \choose m}}\nonumber\\
&=
\frac{d - 1}{2m}\frac{{N - \lfloor N/m \rfloor \choose u - \lfloor N/m \rfloor}}{{N \choose m}}.\label{eq:last-binomial}
\end{align}
Since $u\geq N/2$ and $m \geq 9$, we have
\[
\lfloor N / m \rfloor
\leq
N/m
\leq
2u/9
< u.
\]
Applying this fact and \eqref{eq:proof-d-lower-bound} to \eqref{eq:last-binomial} yields
\begin{align*}
\mathcal{M}_{N,m}^{\mathrm{I}}(\Hyp)
&\geq
\frac{d-1}{16m}.
\end{align*}
\end{proof}

\section{Proofs of upper bounds for TLSII}

\subsection{upper bounds for TLSI lead to upper bounds for TLSII}
\label{subsec:II-to-I}
It is well known that upper bounds for TLSI lead to upper bounds for TLSII 
\cite[Theorem 8.1]{Vap98}. 
This is illustrated in the next result.
\begin{theorem}
\label{thm:erm-upper-setting2}
Consider TLSII.
Let $\Hyp$ be a set of classifiers with VC dimension $2\leq d<\infty$.
Assume that $u \geq 4$ and $u \geq m \geq d - 1$.
Assume the existence of $h^{\star} \in \Hyp$, such that $h^{\star}(X) = Y$ with probability 1 for $(X,Y)\sim P$.
Then for any $\delta \in (0,1)$ with probability at least $1-\delta$ (over the random choices of samples $\Z_m$ and $\Z_u$) for ERM $\hat{h}_m$ it holds that
\[
\err(\hat{h}_m, \Z_u)
\leq
2\frac{d\log(Ne/d) + \log\frac{1}{\delta}}{m}.
\]
\end{theorem}
\begin{proof}
\begin{align*}
\Prob\left\{
\err(\hat{h}_m, \Z_u)
\geq
\epsilon
\right\}
&=
\E\left[
\mathbbm{1}\left\{\err(\hat{h}_m, \Z_u)
\geq
\epsilon\right\}
\right]\\
&=
\E\left[
\frac{1}{N!}\sum_{\pi\in\Sigma_N}
\mathbbm{1}\left\{\err(\hat{h}_m^\pi, \Z_u^\pi)
\geq
\epsilon\right\}
\right]\\
&=
\E\left[
\Prob_\pi\left\{\err(\hat{h}_m^\pi, \Z_u^\pi)
\geq
\epsilon\right\}
\right]\\
&\leq
\left(\frac{Ne}{d}\right)^d e^{-\epsilon m / 2},
\end{align*}
where the last step is due Theorem \ref{thm:erm-upper-setting1}.
\end{proof}
Unfortunately, the tail bound of Theorem~\ref{thm:erm-upper-setting2}, as well as its in-expectation counterpart, are worse than the ones provided by the direct analysis of Theorems \ref{thm:erm-upper-setting2-better} and~\ref{thm:erm-upper-expect-setting2}. 
In particular, we pay a $\log(N)/\log(m)$ factor.

\subsection{Proof of Theorem~\ref{thm:erm-upper-setting2-better}}\label{proof:erm-upper-setting2-better}

Recall that the empirical risk minimizer $\hat{h}_m$ is built without making
use of $\Z_u$. Then, when conditioning on $\Z_m$, the test error
$\err(\hat{h}_m, \Z_u)$ is a sum of i.i.d.\:Bernoulli random variables taking
values in $\{0,1/u\}$. The variance of this sum is
\[
g(\Z_m):=\mathbb{V}\left[ \mathbbm{1}\{\hat{h}_m(X) \neq Y\} \big| \Z_m\right] 
\leq
\Prob\{\hat{h}_m(X) \neq Y \big | \Z_m\}=: f(\Z_m),
\]
On the other hand, the expectation of this sum is
\[
\E\left[\err(\hat{h}_m, \Z_u) \big| \Z_m\right] 
=
\E\left[\frac{1}{u}\sum_{(x,y)\in\Z_u} \mathbbm{1}\{\hat{h}_m(x)\neq y\} \bigg| \Z_m\right] 
=
f(\Z_m).
\]
Using Bernstein's inequality \cite[Theorem 2.10]{BLM13} together with the
previous expectation and variance, we obtain
\begin{align*}
&\Prob\left\{
\err(\hat{h}_m, \Z_u)
\geq
f(\Z_m)
+ \sqrt{\frac{2f(\Z_m)\log\frac{1}{\delta}}{u}} + \frac{2\log\frac{1}{\delta}}{3u}
\right\}\\
&=
\int_{\Z_m}
\Prob\left\{
\err(\hat{h}_m, \Z_u)
\geq
f(\Z_m)
+ \sqrt{\frac{2f(\Z_m)\log\frac{1}{\delta}}{u}} + \frac{2\log\frac{1}{\delta}}{3u} \bigg| \Z_m
\right\}
dP(\Z_m)\\
&\leq
\int_{\Z_m}
\Prob\left\{
\err(\hat{h}_m, \Z_u)
\geq
f(\Z_m)
+ \sqrt{\frac{2g(\Z_m) \log\frac{1}{\delta}}{u}} + \frac{2\log\frac{1}{\delta}}{3u} \bigg| \Z_m
\right\}
dP(\Z_m)\\
&\leq
\int_{\Z_m} \delta\, dP(\Z_m) = \delta.
\end{align*}
On the other hand, using an upper bound of \cite[Problem 12.9]{DGL96}, we get
\[
\Prob\left\{
f(\Z_m)
\geq
\frac{2d\log(m) + \log 2 + \log(1/\delta)}{m}
\right\}
\leq
\delta.
\]
Now we guarantee the success of the previous two events using the union bound.
Then, with probability at least $1-\delta$, it follows that
\begin{align*}
\err(\hat{h}_m, \Z_u)
&\leq
\frac{2d\log(m)  + \log 2 + \log\frac{2}{\delta}}{m}
+ \sqrt{\frac{2 \log\frac{2}{\delta}}{u}\left(\frac{2d\log(m) + \log 2 + \log\frac{2}{\delta}}{m}\right)} + \frac{2\log\frac{2}{\delta}}{3u}\\
&\leq
\frac{6d\log(m) + 3\log 2 + 3\log\frac{2}{\delta}}{2m}
+
\frac{5\log\frac{2}{\delta}}{3u},
\end{align*}
where the last inequality is due $\sqrt{ab}\leq\frac{a+b}{2}$ for all $a,b,\in\mathbb{R}$.
We can repeat the same argument for $\tilde{h}_m$ and use the upper bound presented in \cite[Theorem 2]{H15}.

\section{Proofs of lower bounds for TLSII}
\subsection{Proof of Theorem \ref{thm:lower-exp-setting2}}
\label{section:lower-exp-setting2}
We start by observing that we can lower bound the minimax value
\eqref{eq:second_line} by restricting the family of distributions running over
the supremum. 
In particular, we choose $P$ such that $\Z_N:=\{(X_i,Y_i)\}_{i=1}^m \cup
\{(X_i,Y_i)\}_{i=m+1}^{m+u}$ will concentrate around the datasets designed in
Section~\ref{sec:proofs-lower-i}.
Namely, choose a set of distinct points $x_1,\dots,x_d$ shattered by $\Hyp$, pair such set to any
binary string $b:=(b_1,\dots,b_d)$, 
assume $m \geq d-1$,
and construct 
\begin{equation*}
P_0(X,Y) = \begin{cases}
  1/m &\mbox{if } (X,Y) = (x_i, b_i) \mbox{ for some } i=1,\ldots,d-1,\\ 
  1 - \frac{d-1}{m} & \mbox{if } (X,Y) = (x_d,b_d),\\
  0 & \mbox{otherwise.}
\end{cases} 
\end{equation*}

Since $x_1, \ldots, x_d$ are shattered by $\Hyp$, the inner infimum from
\eqref{eq:second_line} is equal to zero. Then,
\begin{align*}
\mathcal{M}^{\mathrm{II}}_{N,m}(\Hyp)
\geq
\inf_{h_m}\sup_b
\int_{\X\times\Y}
\frac{1}{u}\sum_{i=m+1}^{m+u} \mathbbm{1}\{ h_m(X_i) \neq Y_i \}\; d\,P_0(X_1,Y_1)\dots d\,P_0(X_{m+u},Y_{m+u}),
\end{align*}
where now $(X_1,Y_1),\dots,(X_{m+u}, Y_{m+u}) \sim P_0$.
Since the distribution of
\[
\sum_{i=m+1}^{m+u} \mathbbm{1}\{ h_m(X_i) \neq Y_i \}
\]
do not change over random permutations of the set $\{(X_i,Y_i)\}_{i=1}^{m+u}$, it follows that
\begin{align*}
&\mathcal{M}^{\mathrm{II}}_{N,m}(\Hyp)\\
&\geq
\inf_{h_m}
\sup_b
\int_{\X\times\Y}
\frac{1}{u}\sum_{i=m+1}^{m+u} \mathbbm{1}\{ h_m(X_i) \neq Y_i \}\; d\,P_0(X_1,Y_1)\dots d\,P_0(X_{m+u},Y_{m+u})\\
&=
\inf_{h_m}
\sup_b
\int_{\X\times\Y}
\E_{\pi}\left[
\frac{1}{u}\sum_{i=m+1}^{m+u} \mathbbm{1}\{ h_m^{\pi}(X_{\pi(i)}) \neq Y_{\pi(i)} \}\right]
d\,P_0(X_1,Y_1)\dots d\,P_0(X_{m+u},Y_{m+u})
\end{align*}
where 
\[
h_m^{\pi}:= h_m\left( \{(X_{\pi(i)},Y_{\pi(i)})\}_{i=1}^m,\{X_{\pi(j)}\}_{j=m+1}^{u+m} \right) \in \Hyp
\]
and $\pi$ is uniformly distributed over $\Sigma_N$. Let $B$ be uniformly
distributed over $\{0,1\}^d$.  We use the probabilistic method to lower bound
the supremum over $b$ with the average over $B$:
\begin{align*}
&\mathcal{M}^{\mathrm{II}}_{N,m}(\Hyp)\\
&\geq
\inf_{h_m}
\E_B\left[
\int_{\X\times\Y}
\E_{\pi}\left[
\frac{1}{u}\sum_{i=m+1}^{m+u} \mathbbm{1}\{ h_m^{\pi}(X_{\pi(i)}) \neq Y_{\pi(i)} \}\right]
d\,P_0(X_1,Y_1)\dots d\,P_0(X_{m+u},Y_{m+u})\right].
\end{align*}

For $\X_N:=\{X_i\}_{i=1}^N$ and  $j\in \{1,\ldots,d\}$ write $i_j(\X_N)$ to
denote the number of times that $x_j$ appears in $\X_N$, and write
$k_j(\X_N,\pi)$ to denote the number of times that $x_j$ appears in
$\{X_{\pi(i)}\}_{i=m+1}^{m+u}$.
In words, $i_j(\X_N)$ is the number of times that the input $x_j$ appears in
the union $\X_N$ of the training and test sets as a result of sampling from
distribution $P_0$, and $k_j(\X_N,\pi)$ is the number of times that the same
input appears in the test subset of $\X_N$, as specified by the permutation
$\pi$.

Using the previous notations, and for any fixed sample
$\{(X_i,Y_i)\}_{i=1}^{m+u}$, permutation $\pi\in\Sigma_N$, and binary string
$B\in\{0,1\}^d$, write
\begin{align*}
\sum_{i=m+1}^{m+u} \mathbbm{1}\{ h_m^{\pi}(X_{\pi(i)}) \neq Y_{\pi(i)} \}
&=
\sum_{i=1}^d k_j(\X_N,\pi) \cdot \mathbbm{1}\{h_m^{\pi}(x_i) \neq B_i\}\\
&\geq
\sum_{i=1}^d i_j(\X_N) \cdot \mathbbm{1}\{k_j(\X_N,\pi) = i_j(\X_N)\} \cdot \mathbbm{1}\{h_m^{\pi}(x_i) \neq B_i\}
\end{align*}
and consequently 
\begin{align*}
\mathcal{M}^{\mathrm{II}}_{N,m}(\Hyp)
&\geq
\inf_{h_m}
\E_B
\E_{\X_N}
\E_{\pi}\left[
\frac{1}{u}\sum_{j=1}^{d}i_j(\X_N) \cdot \mathbbm{1}\{i_j(\X_N) = k_j(\X_N,\pi)\} \cdot \mathbbm{1}\{ h_m^{\pi}(x_j) \neq B_j \}\right],
\end{align*}
where $\X_N \sim P_0^{m+u}$. Rearranging expectations yields
\begin{align*}
\mathcal{M}^{\mathrm{II}}_{N,m}(\Hyp)
&\geq
\inf_{h_m}
\E_{\X_N}
\E_{\pi}\left[
\frac{1}{u}\sum_{j=1}^{d}\mathbbm{1}\{i_j(\X_N) = k_j(\X_N,\pi)\} \cdot i_j(\X_N) \cdot
\E_B
\left[
\mathbbm{1}\{ h_m^{\pi}(x_j) \neq B_j \}
\right]\right]\\
&=
\frac{1}{2u}\sum_{j=1}^{d}
\E_{\X_N}
\E_{\pi}\left[
\mathbbm{1}\{i_j(\X_N) = k_j(\X_N,\pi)\} \cdot i_j(\X_N)
\right],
\end{align*}
because $h_m^{\pi}$ is independent from $B_j$ if $k_j(\X_N,\pi) = i_j(\X_N)$.
Also, since $i_j(\X_N)$ is independent from $\pi$, we have that 
\begin{align*}
\mathcal{M}^{\mathrm{II}}_{N,m}(\Hyp)
&\geq
\frac{1}{2u}\sum_{j=1}^{d}
\E_{\X_N}\left[
i_j(\X_N) \cdot
\Prob_{\pi}\bigl\{
i_j(\X_N) = k_j(\X_N,\pi)
\bigr\}
\right]\\
&=
\frac{1}{2u}\sum_{j=1}^{d}
\E_{\X_N}\left[
i_j(\X_N) \cdot
\frac{ {N - i_j(\X_N) \choose u - i_j(\X_N)}}{{N\choose u}}
\right],
\end{align*}
where the identity is due $k_j(\X_N, \pi)$ being a random variable following
a hypergeometric distribution with parameters $\bigl(N, i_j(\X_N), u\bigr)$.
By realizing that $\bigl( i_1(\X_N), \dots, i_d(\X_N)\bigr)$ follows a
multinomial distribution with parameters $\bigl(N, 1/m,\dots,1/m, 1 -
(d-1)/m\bigr)$, we obtain
\begin{align*}
\mathcal{M}^{\mathrm{II}}_{N,m}(\Hyp)
&\geq
\frac{1}{2u}\sum_{j=1}^{d}
\E_{\X_N}\left[
i_j(\X_N) \cdot
\frac{ {N - i_j(\X_N) \choose m}}{{N\choose m}}
\right]\\
&=
\frac{1}{2u}\sum_{j=1}^{d-1}
\sum_{k=0}^N {N \choose k}\left(\frac1m\right)^k \left(1 - \frac1m\right)^{N-k} k \frac{ {N - k \choose m}}{{N\choose m}}\\
&+
\frac{1}{2u}
\sum_{k=0}^N {N \choose k}\left(1 - \frac{d-1}{m}\right)^k \left(\frac{d-1}{m}\right)^{N-k} k \frac{ {N - k \choose m}}{{N\choose m}}
\\
&\geq
\frac{1}{2u}\sum_{j=1}^{d-1}
\sum_{k=0}^u k {u \choose k}\left(\frac1m\right)^k \left(1 - \frac1m\right)^{N-k} \\
&=
\frac{d-1}{2u}\left(1 - \frac{1}{m}\right)^m\frac{u}{m}\\
&\geq
\frac{d-1}{2em}\left(1 - \frac{1}{m}\right).
\end{align*}
where the equality follows from expanding the expectation formula of the
Binomial distribution with parameters $(N, 1/m)$ for $j=1,\ldots,d-1$ and
parameters $(N, 1-(d-1)/m)$ for $j=d$, the second inequality follows from
discarding the last term of the sum, simplifying the binomial coefficients, and
truncating the sum.  The last equality is due applying the expected value $u/m$
of $(d-1)$ $\text{Binomial}(u,1/m)$ random variables, yielding an extra
$(1-1/m)^m$ extra factor. The last factor is finally lower bounded by
$e^{-1}$.

\subsection{Proof of Theorem~\ref{thm:minimax_hp-II}}
\label{sect:proof-minimax_hp-II}
Let $x_1,\dots,x_d$ be a set of distinct points shattered by $\Hyp$, and let $\vec
b\in\{0,1\}^d$ be a binary string.
Fix a positive $p \leq 1/(d-1)$. Define the probability distribution
\begin{equation*}
  P(X,Y) = \begin{cases}
  p &\mbox{if } (X,Y) = (x_i, b_i) \mbox{ for some } i=1,\ldots,d-1,\\ 
  1 - (d-1)p & \mbox{if } (X,Y) = (x_d,b_d),\\
  0 & \mbox{otherwise.}
\end{cases} 
\end{equation*}

We will denote $\X_m:= \{X_1,\dots, X_m\}$ and $\X_u:=\{X_{m+1},\dots,X_{m+u}\}$.
Recall that $\{(X_i,Y_i)\}_{i=1}^{m+u}$ is an i.\,i.\,d.\:sample from $P$.
For any $\X_u$ and $i\in \{1,\dots, d\}$ let $k_i(\X_u)$ count a number of times an input $x_i$ appeared in $\X_u$.
Then, we have
\begin{align*}
\inf_{h_m}
\sup_P
\Prob\left\{
\err(h_m,\Z_u) \geq \epsilon
\right\}
&\geq
\inf_{h_m}
\sup_{p,\vec b}
\Prob\left\{
\err(h_m,\Z_u) \geq \epsilon
\right\}
\\
&=
\inf_{h_m}
\sup_{p,\vec b}
\Prob\left\{ \sum_{(x,y)\in \Z_u} \mathbbm{1}\{h_m(x)\neq y\} \geq \epsilon u \right\}\\
&=
\inf_{h_m}
\sup_{p,\vec b}
\Prob\left\{
\sum_{i=1}^d \mathbbm{1}\{h_m(x_i)\neq b_i\} k_i(\X_u)
\geq \epsilon u
\right\}\\
&\geq
\inf_{h_m}
\sup_{p}
\Prob\left\{
\sum_{i=1}^d \mathbbm{1}\{h_m(x_i)\neq B_i\} k_i(\X_u)
\geq \epsilon u
\right\},
\end{align*}
where the last inequality lower bounds the supremum over $\vec b$ with the
expected value over $(B_1,\dots,B_d)$, and $B$ is a random binary string
uniformly distributed on $\{0,1\}^d$. 
Throwing away summands for which $x_i \in \X_m$ we arrive at the following lower bound:
\begin{align}
\label{eq:s2-proof-general}
&
\inf_{h_m}
\sup_P
\Prob\left\{
\err(h_m,\Z_u) \geq \epsilon
\right\}
\geq
\inf_{h_m}
\sup_p
\Prob\left\{
\sum_{i=1}^d 
\mathbbm{1}\{h_m(x_i) \neq B_i\}
\mathbbm{1}\{ x_i \not\in \X_m\}
k_i(\X_u)
\geq
\epsilon u
\right\}.
\end{align}
Equation~\ref{eq:s2-proof-general} is the starting point to prove the two
separate statements comprising our result.

\subsubsection{Statement \ref{item:st1-s2}}
Let $\epsilon m \leq \frac{d-1}{21}$ and
further lower bound \eqref{eq:s2-proof-general} by ignoring the term corresponding to $i=d$:
\[
\inf_{h_m}
\sup_P
\Prob\left\{
\err(h_m,\Z_u) \geq \epsilon
\right\}
\geq
\inf_{h_m}
\sup_p
\Prob\left\{
\sum_{i=1}^{d-1} 
\mathbbm{1}\{h_m(x_i) \neq B_i\}
\mathbbm{1}\{ x_i \not\in \X_m\}
k_i(\X_u)
\geq
\epsilon u
\right\}.
\]
Denote
\[
Z=\sum_{i=1}^{d-1}
\mathbbm{1}\{h_m(x_i) \neq B_i\}
\mathbbm{1}\{ x_i \not\in \X_m\}
k_i(\X_u).
\]
Then,
\begin{align*}
\E\left[
Z
\right]
&=
\sum_{i=1}^{d-1} 
\Prob\left\{
h_m(x_i) \neq B_i
\cap
x_i \not\in \X_m
\right\}
\E\left[k_i(\X_u)
\right]\\
&=
\sum_{i=1}^{d-1} 
\Prob\left\{
h_m(x_i) \neq B_i
\big|
x_i \not\in \X_m
\right\}
\Prob\left\{
x_i \not\in \X_m
\right\}
u p\\
&=
\sum_{i=1}^{d-1} 
\frac{1}{2}
(1 - p)^m
u p
=
(d-1)\frac{up}{2}(1-p)^m,
\end{align*}
where the previous follows because $\X_u$ and $\X_m$ are independent, and
$\bigl(k_1(\X_u),\dots, k_d(\X_u)\bigr)$ follows a multinomial distribution of
$u$ trials and probabilities $\bigl(p,\dots,p,1-(d-1)p\bigr)$.
We can rewrite
\begin{align*}
\Prob\left\{
Z
\geq
\epsilon u
\right\}
&=
1
-
\Prob\left\{
-Z
+
\E[Z]
>
\E[Z]-\epsilon u
\right\}
\\
&=
1
-
\Prob\left\{
-Z+\E[Z]
>
(d-1)\frac{up}{2}(1-p)^m-\epsilon u
\right\}\\
&\geq
1
-
\Prob\left\{
-Z
+\E[Z]
>
(d-1)\frac{up}{2}(1-p)^m- \frac{(d-1)u}{21m}
\right\}.
\end{align*}

Next, we apply the Chebyshev-Cantelli inequality \cite[Theorem A.17]{DGL96} to
lower bound the previous expression. First, we simplify the probability
threshold used in the inequality. To this end, set $p = \frac{1}{2m}$, and
assume $m \geq \max\{(d-1)/2, 10\}$. In particular, this choice guarantees
$p\leq 1/(d-1)$, and provides
\begin{align*}
(d-1)\frac{up}{2}(1-p)^m- \frac{(d-1)u}{21m}
&=
\frac{u(d-1)}{4m}\left(1-\frac{1}{2m}\right)^m - \frac{(d-1)u}{21m}\\
&=
\frac{u(d-1)}{4m}\left(\left(1-\frac{1}{2m}\right)^{2m-1}\left(1-\frac{1}{2m}\right)\right)^{\frac{1}{2}} - \frac{(d-1)u}{21m}\\
&\geq
\frac{u(d-1)}{4m}\sqrt{\frac{19}{20e}} - \frac{(d-1)u}{21m} 
=  
C_0\frac{u(d-1)}{m}> 0,
\end{align*}
where the last inequality uses $m\geq 10$, $(1 - 1/x)^{x-1} \geq e^{-1}$, valid for all
$x\geq 1$, and introduces the notation 
\[
C_0 := \frac{1}{4}\sqrt{\frac{19}{20e}} - \frac{1}{21} > 0.
\]
In order to apply Chebyshev-Cantelli inequality we also need to upper bound the variance $\mathbb{V}[Z]$:
\begin{align*}
\mathbb{V}[Z]
&\leq
\E\left[
\left(
\sum_{i=1}^{d-1} 
\mathbbm{1}\{h_m(x_i) \neq B_i\}
\mathbbm{1}\{ x_i \not\in \X_m\}
k_i(\X_u)
\right)^2\right]\\
&\leq
\E\left[
\left(
\sum_{i=1}^{d-1} 
k_i(\X_u)
\right)^2\right]
=
\mathbb{V}\left[
\sum_{i=1}^{d-1} 
k_i(\X_u)
\right]
+
\left(\E\left[
\sum_{i=1}^{d-1} 
k_i(\X_u)
\right]\right)^2\\
&=
u(d-1)p\bigl(1 - (d-1)p\bigr)
+
u^2(d-1)^2p^2,
\end{align*}
where the previous follows because 
\[
\sum_{i=1}^{d-1} 
k_i(\X_u)
\sim
\mathrm{Binom}\bigl(u, (d-1)p\bigr).
\]
Using the previous probability threshold and variance, we apply the Chebyshev-Cantelli inequality as
\begin{align*}
\Prob\left\{
Z
\geq
\epsilon u
\right\}
&\geq
1
-
\frac{\mathbb{V}[Z]}{\mathbb{V}[Z]+ C_0^2\frac{u^2(d-1)^2}{m^2}}\\
&\geq
1
-
\frac{\frac{m}{2u(d-1)}\left(1 - \frac{d-1}{2m}\right)+\frac{1}{4}}
{\frac{m}{2u(d-1)}\left(1 - \frac{d-1}{2m}\right)+\frac{1}{4}+ C_0^2}\\
&\geq
1
-
\frac{\frac{1}{2(d-1)}+\frac{1}{4}}
{\frac{1}{2(d-1)}+\frac{1}{4}+ C_0^2}\\
&\geq
1
-
\frac{\frac{3}{4}}
{\frac{3}{4}+ C_0^2}
\geq
\frac{1}{80},
\end{align*}
where we used $u\geq m$, $d\geq 2$, and the numerical value of $C_0$.
This concludes the proof of the first statement.

\subsubsection{Statement \ref{item:st2-s2}}
Note that if $\epsilon \neq 0$, then we can assume $\epsilon \geq 1/u$, because
$\err(h_m,\Z_u)$ can not take values in $(0,1/u)$.  We start by rewriting
\eqref{eq:s2-proof-general} as 
\begin{align*}
\sup_p
\sum_{K=0}^{d-1}
\Prob\left\{
\sum_{j\in J(\X_m)}
B_j k_j(\X_u)
\geq
\epsilon u
\bigg|
|J(\X_m)| = K
\right\}
\Prob\{|J(\X_m)| = K \}.
\end{align*}
This expression calls for four remarks.
First, $J(\X_m):=\{j=1,\dots,d\colon x_j \not \in \X_m\}$ are the indices of
the inputs not appearing in the training set $\Z_m$.
Second, the upper limit of the previous sum is $d-1$, since at least one of the
$d$ inputs $x_1,\dots,x_d$ appears in $\X_m$ and also we assumed $m \geq d -1$.
Third, for any $j\in J(\X_m)$, the random variable $\mathbbm{1}\{h_m(x_j)\neq
B_j\}$ follows a Bernoulli distribution with parameter $1/2$.
Fourth, for any two different $i,j\in J(\X_m)$, the random variables
$\mathbbm{1}\{h_m(x_i)\neq B_i\}$ and $\mathbbm{1}\{h_m(x_j)\neq B_j\}$ are
independent (for more details, revisit the proof of Theorem
\ref{thm:minimax_hp}).
Then, the sum $\sum_{j\in J(\X_m)} \mathbbm{1}\{h_m(x_j) \neq B_j\}k_j(\X_u)$
is a sum of $|J(\X_m)|$ independent $1/2$ Bernoulli random variables, where the $j$th of them is weighted by $k_j(\X_u)$. 

Next, we specify which $K$ inputs $\{x_{i_1},\dots,x_{i_K}\} \subset
\{x_1,\dots,x_d\}$ do not appear in the training set $\Z_m$.  For any set of
indices $I\subseteq\{1,\dots,d\}$, let $E(I)$ denote all sets of inputs $\X_m$
satisfying $x_i\not \in \X_m$ if $i\in I$, and $x_j\in\X_m$ if
$j\in\{1,\dots,d\}\setminus I$.  Then, for any two subsets
$I_1,I_2\subseteq\{1,\dots,d-1\}$ of equal cardinality $|I_1| = |I_2|$, it
follows that 
\[
\Prob_{\X_m}\{ E(I_1)\}
=
\Prob_{\X_m}\{ E(I_2)\},
\]
since inputs $x_1,\dots,x_{d-1}$ are equiprobable for our choice of distribution $P$.
By ignoring the cases where $x_d$ does not appear in the training set, we get 
\begin{align*}
&
\inf_{h_m}
\sup_P
\Prob\left\{
\err(h_m,\Z_u) \geq \epsilon
\right\}\\
&\geq
\sup_p
\sum_{K=0}^{d-1}
{d-1 \choose K}
\Prob\left\{
Z\bigl(K,\vec k(\X_u)\bigr)
\geq
\epsilon u
\right\}
\Prob\bigl\{
E(\{1,\dots,K\})
\bigr\},
\end{align*}
where $\vec k(\X_u) := \bigl(k_1(\X_u),\dots,k_d(\X_u)\bigr)$, and $Z(K,\vec a)
= \sum_{j=1}^K \vec a_j B_j$ is a weighted sum of i.i.d. Bernoulli random variables
$\{B_i\}_{i=1}^d$ with parameter $1/2$, for some $\vec a \in \R^d_+$.  The
binomial coefficient ${d-1 \choose K}$ accounts for the number of subsets 
of $\{x_1,\dots,x_{d-1}\}$ with $K$ elements.

Note that  
\begin{align}
\Prob\bigl\{
E(\{1,\dots,K\})
\bigr\}
&\geq
p^{d-K-1}
\bigl( 1 - (d-1)p\bigr)
(1 - Kp)^{m-d+K}\label{eq:three-factors}\\
&\geq
p^{d-K-1}
\bigl( 1 - (d-1)p\bigr)^{m-d+K+1},\nonumber
\end{align}
holds because $K\leq d -1$, each of the inputs
$x_{K+1},\dots,x_d$ appears at least once in $\X_m$ (see the first two factors
of \eqref{eq:three-factors}), and none of the inputs $x_1,\dots,x_K$ appears in
$\X_m$ (see the third factor in \eqref{eq:three-factors}).
Using this expression, our lower bound becomes 
\[
\sup_p
\sum_{K=0}^{d-1}
{d-1 \choose K}
\Prob\left\{
Z\bigl(K,\vec k(\X_u)\bigr)
\geq
\epsilon u
\right\}
p^{d-K-1}
\bigl( 1 - (d-1)p\bigr)^{m-d+K+1}.
\]

We further lower bound by truncating the start of the sum, as in
\begin{equation}
\sup_p
\sum_{K=\lceil(d-1)/2\rceil}^{d-1}
{d-1 \choose K}
\Prob\left\{
Z\bigl(K,\vec k(\X_u)\bigr)
\geq
\epsilon u
\right\}
p^{d-K-1}
\bigl( 1 - (d-1)p\bigr)^{m-d+K+1}.
\label{eq:intermediate3}
\end{equation}

Next, we are interested in applying the  Chebyshev-Cantelli inequality
\cite[Theorem A.17]{DGL96} to the random variable $Z\bigl(K,\vec k(\Z_u)\bigr)$
in \eqref{eq:intermediate3}. To this end, we must first compute its expectation
and variance. We start by noticing that the random variable
\[
\bigl(k_1(\X_u),\dots, k_d(\X_u)\bigr)\]
follows a multinomial distribution of $u$ trials and probabilities
$\bigl(p,\dots,p,1-(d-1)p\bigr)$.  This implies
\begin{equation}
\E\bigl[Z\bigl(K,\vec k(\Z_u)\bigr)\bigr] = \frac{K}{2} up,\label{eq:expectation3}
\end{equation}
and by definition we have
\begin{align*}
\mathbb{V}\bigl[Z\bigl(K,\vec k(\Z_u)\bigr)\bigr]
&=
\E\bigl[Z^2\bigl(K,\vec k(\Z_u)\bigr)\bigr]
-
\frac{K^2}{4}u^2 p^2.
\end{align*}
Since $Z$ depends on the Bernoulli random variables $B := B_1,\dots,B_d$,
conditioning on $B$ produces 
\[
\E\bigl[Z^2\bigl(K,\vec k(\Z_u)\bigr)\bigr]
=
\E\left[
\E\left[Z^2\bigl(K,\vec k(\Z_u)\bigr)\bigg| \sum_{i=1}^K B_i\right]
\right].
\]
For any index set $I\subseteq \{1,\dots,d-1\}$, it follows from the properties of multinomial distribution that
\[
\sum_{i\in I}  k_i(\X_u)
\sim
\mathrm{Binom}(u, |I|p).
\]
Let $V = \sum_{i=1}^K B_i$. Then,
\begin{align*}
\E\left[
\E\left[Z^2\bigl(K,\vec k(\Z_u)\bigr)\big| V\right]
\right]
&=
\E\left[
\E\left[
\bigl(\mathrm{Binom}(u, Vp)\bigr)^2
\big| V\right]
\right]\\
&=
\E\left[
\mathbb{V}\left[\mathrm{Binom}(u, Vp)
\big| V\right] + \left(\E\left[\mathrm{Binom}(u, Vp)
\big| V\right]\right)^2
\right]
\\
&=
\E\left[
uVp(1-Vp) + \left(\E\left[\mathrm{Binom}(u, Vp)
\big| V\right]\right)^2
\right]\\
&=
\E\left[
uVp(1-Vp) + u^2V^2p^2
\right].
\end{align*}
Noting that
\[
\E\left[V^2\right] = \E\left[ \left(\sum_{i=1}^K B_i\right)^2\right]
=
K + \frac{K^2 - K}{4} = \frac{K(K+1)}{4}
\]
we get
\begin{align}
\mathbb{V}\bigl[Z\bigl(K,\vec k(\Z_u)\bigr)\bigr]
&=
up\frac{K}{2}
-
up^2\frac{K(K+1)}{4}
+
u^2p^2\frac{K(K+1)}{4}
-u^2p^2\frac{K^2}{4}\nonumber\\
&=
up\frac{K}{2}
-
up^2\frac{K(K+1)}{4}
+
u^2p^2\frac{K}{4}\nonumber\\
&=
\frac{upK}{2}\left(
1 - p\frac{K+1}{2} + \frac{up}{2}\label{eq:variance3}
\right).
\end{align}
We are now ready to apply the Chebyshev-Cantelli inequality \cite[Theorem
A.17]{DGL96} using the expectation \eqref{eq:expectation3} and the variance
\eqref{eq:variance3}. In particular, 
\begin{align*}
\Prob\left\{
Z\bigl(K,\vec k(\Z_u)\bigr)
\geq
\epsilon u
\right\}
&=
1 - \Prob\left\{
-Z\bigl(K,\vec k(\Z_u)\bigr) + \frac{K}{2}up
>
\frac{K}{2}up - 
\epsilon u
\right\}\\
&\geq
1
-
\frac{\frac{upK}{2}\left(
1 - p\frac{K+1}{2} + \frac{up}{2}
\right)}
{\frac{upK}{2}\left(
1 - p\frac{K+1}{2} + \frac{up}{2}
\right) + 
\left(\frac{K}{2}up - 
\epsilon u\right)^2}\\
&=
1
-
\frac{\frac{pK}{2}\left(
1 - p\frac{K+1}{2} + \frac{up}{2}
\right)}
{\frac{pK}{2}\left(
1 - p\frac{K+1}{2} + \frac{up}{2}
\right) + 
\left(\frac{K}{2}p - 
\epsilon\right)^2u}
\end{align*}
as long as
\[
\frac{K}{2}p  \geq
\epsilon.
\]

To guarantee this, set $p = \frac{16 \epsilon}{d-1}$, and $\epsilon\leq
1/16$ (which was also needed to satisfy $p\leq 1/(d-1)$):
\[
\frac{K}{2}p = \frac{KD\epsilon}{2(d-1)} \geq \frac{16\epsilon}{4} = 4\epsilon > \epsilon.
\]
Using this choice, continue lower bounding as 
\begin{align*}
\Prob\left\{
Z\bigl(K,\vec k(\Z_u)\bigr)
\geq
\epsilon u
\right\}
&\geq
1
-
\frac{\frac{16\epsilon K}{2(d-1)}\left(
1 - \frac{16\epsilon(K+1)}{2(d-1)} + \frac{16u\epsilon}{2(d-1)}
\right)}
{\frac{16\epsilon K}{2(d-1)}\left(
1 - \frac{16\epsilon(K+1)}{2(d-1)} + \frac{16u\epsilon}{2(d-1)}
\right) + 
\left(3\epsilon\right)^2u}\\
&=
1
-
\frac{\frac{8K}{d-1}\left(
1 - \frac{8\epsilon(K+1)}{d-1} + \frac{8u\epsilon}{d-1}
\right)}
{\frac{8 K}{d-1}\left(
1 - \frac{8\epsilon(K+1)}{d-1} + \frac{8u\epsilon}{d-1}
\right) + 
9u\epsilon}\\
&\geq
1
-
\frac{8\left(
1 - 4\epsilon + \frac{8u\epsilon}{d-1}
\right)}
{8\left(
1 - 4\epsilon + \frac{8u\epsilon}{d-1}
\right) + 
9u\epsilon}\\
&=
1
-
\frac{8\left(
\frac{1}{\epsilon u} - \frac{4}{u} + \frac{8}{d-1}
\right)}
{8\left(
\frac{1}{\epsilon u} - \frac{4}{u} + \frac{8}{d-1}
\right) + 
9},
\end{align*}
where the last inequality is due to $\lceil (d-1)/2\rceil \leq K\leq d-1$, and
the fact that $x\mapsto\frac{x}{x + a}$ is an increasing function for $x,a\geq
0$.  By noting that $1/(\epsilon u) \leq 1$, we get
\[
\Prob\left\{
Z\bigl(K,\vec k(\Z_u)\bigr)
\geq
\epsilon u
\right\}
\geq
1
-
\frac{8\left(
1 + 8
\right)}
{8\left(
1 + 8
\right) + 
9}
=
\frac{1}
{9}.
\]

Plugging this constant into \eqref{eq:intermediate3} yields
\begin{align}
&\inf_{h_m}
\sup_P
\Prob\left\{
\err(h_m,\Z_u) \geq \epsilon
\right\}\nonumber\\
&\geq
\frac{1}{9} \sum_{K=\lceil(d-1)/2\rceil}^{d-1}
{d-1 \choose K}
\left(\frac{16\epsilon}{d-1}\right)^{d-K-1}
\bigl( 1 - 16\epsilon\bigr)^{m-d+K+1}\nonumber\\
&\geq
\frac{1}{9} \sum_{K=\lceil(d-1)/2\rceil}^{d-1}
{d-1 \choose K}
\left(\frac{16\epsilon}{d-1}\right)^{d-K-1}
\bigl( 1 - 16\epsilon\bigr)^{m}\nonumber\\
&\geq
\frac{1}{9} e^{-\frac{16m\epsilon}{1-16\epsilon}}
\left(\frac{16\epsilon}{d-1}\right)^{d-1}
\sum_{K=\lceil(d-1)/2\rceil}^{d-1}
{d-1 \choose K}
\left(\frac{d-1}{16\epsilon}\right)^{K}\nonumber\\
&\geq
\frac{1}{9} e^{-32m\epsilon}
\left(\frac{16\epsilon}{d-1}\right)^{d-1}
\sum_{K=\lceil(d-1)/2\rceil}^{d-1}
{d-1 \choose K}
\left(\frac{d-1}{16\epsilon}\right)^{K},\label{eq:intermediate4}
\end{align}
where we lower-bounded exponents, and the third inequality is due to $1-x\geq
e^{-x/(1-x)}$, $\epsilon \leq 1/32$.
Note that $(d-1)/(16\epsilon) \geq 1$ and that 
\[
d - 1 - K \leq K
\]
holds for $K\in\{\lceil(d-1)/2\rceil,\dots, d-1\}$. Then,
\begin{align*}
\sum_{K=\lceil(d-1)/2\rceil}^{d-1}
{d-1 \choose K}
\left(\frac{d-1}{16\epsilon}\right)^{K}
&=
\sum_{K=\lceil(d-1)/2\rceil}^{d-1}
{d-1 \choose d-1 - K}
\left(\frac{d-1}{16\epsilon}\right)^{K}\\
&\geq
\sum_{K=\lceil(d-1)/2\rceil}^{d-1}
{d-1 \choose d-1 - K}
\left(\frac{d-1}{16\epsilon}\right)^{d-1-K}\\
&=
\sum_{K=0}^{d-1-\lceil(d-1)/2\rceil}
{d-1 \choose K}
\left(\frac{d-1}{16\epsilon}\right)^{K}\\
&\geq
\sum_{K=0}^{\lceil(d-1)/2\rceil-1}
{d-1 \choose K}
\left(\frac{d-1}{16\epsilon}\right)^{K},
\end{align*}
where the last inequality uses the fact that, for any integer $d\geq 2$, it
follows that 
\[
d - \left\lceil \frac{d-1}{2}\right\rceil
\geq
\left\lceil \frac{d-1}{2}\right\rceil.
\]
Next, we apply the Binomial theorem
\[
\sum_{K=0}^{d-1}
{d-1 \choose K}
\left(\frac{d-1}{16\epsilon}\right)^{K}
=
\left( 1 + \frac{d-1}{16\epsilon}\right)^{d-1}
\]
to obtain
\[
\sum_{K=\lceil(d-1)/2\rceil}^{d-1}
{d-1 \choose K}
\left(\frac{d-1}{16\epsilon}\right)^{K}
\geq
\frac{1}{2}\left( 1 + \frac{d-1}{16\epsilon}\right)^{d-1}.
\]
Plugging this last result into \eqref{eq:intermediate4} produces 
\begin{align*}
\inf_{h_m}
\sup_P
\Prob\left\{
\err(h_m,\Z_u) \geq \epsilon
\right\}
&\geq
\frac{1}{18} e^{-32m\epsilon}
\left(\frac{16\epsilon}{d-1}\right)^{d-1}
\left( 1 + \frac{d-1}{16\epsilon}\right)^{d-1}\\
&=
\frac{1}{18} e^{-32m\epsilon}
\left( 1 + \frac{16\epsilon}{d-1}\right)^{d-1}\\
&\geq
\frac{1}{18} e^{-32m\epsilon}.
\end{align*}

\section{Proofs from Section \ref{subsection:relations}}
Recall that $h_m$ is used to denote learning algorithms based both on labeled training sample $\Z_m$ and unlabeled points $\X_u$, while $h^0_m$ denotes supervised learning algorithms based only on $\Z_m$.
\label{proof:relations}
\subsection{Proof of Theorem \ref{thm:relation-II-SL}}
\label{proof:relation-II-SL}
First we will prove the first inequality of \eqref{eq:relations-expect}.
We have
\begin{align*}
\mathcal{M}^{\mathrm{II}}_{N,m}(\Hyp)&:=\inf_{h_m}
\sup_{P}
\E\left[
\err(h_m,\Z_u)
\right]\\
&=
\inf_{h_m}
\sup_{P}
\Bigl[
\E\left[
\err(h_m,\Z_u) - L(h_m)
\right]
+
\E\left[
L(h_m)
\right]
\Bigr]\\
&\leq
\inf_{h_m}
\sup_{P}
\E\left[
\err(h_m,\Z_u) - L(h_m)
\right]
+
\inf_{h_m}
\sup_{P}
\E\left[
L(h_m)
\right],
\end{align*}
where we used $\sup(a+b) \leq \sup a + \sup b$.
Obviously, 
\begin{align*}
\inf_{h_m}
\sup_{P}
\E\left[
\err(h_m,\Z_u) - L(h_m)
\right]
&\stackrel{(i)}{\leq}
\inf_{h_m^0}
\sup_{P}
\E\left[
\err(h_m^0,\Z_u) - L(h_m^0)
\right]\\
&\stackrel{(ii)}{=}
\inf_{h_m^0}
\sup_{P}
\E\left[
\E\left[
\err(h_m^0,\Z_u) - L(h_m^0)
\Big| \Z_m
\right]
\right] = 0,
\end{align*}
where (i) is because $h_m$ is allowed to ignore $\X_u$ and
(ii) is uses the fact that, when conditioned on $\Z_m$, $\err(h^0_m, \Z_u)$ is an average of i.i.d.\:Bernoulli random variables with parameters $L(h_m^0)$.
We conclude that
\[
\mathcal{M}^{\mathrm{II}}_{N,m}(\Hyp)
\leq
\inf_{h_m}
\sup_{P}
\E\left[
L(h_m)
\right]
=
\mathcal{M}^{\mathrm{SSL}}_{N,m}(\Hyp).
\]

For the second inequality of \eqref{eq:relations-expect} we notice that
\[
\inf_{h_m} 
\sup_P
\E\left[
L( h_m)\right]
\leq
\inf_{h^0_m} 
\sup_P
\E\left[
L( h^0_m)\right].
\]

Next we turn to the first inequality of \eqref{eq:relations-prob}.
\begin{align}
\notag
\mathcal{M}^{\mathrm{II}}_{\epsilon,N,m}(\Hyp)
&:=
\inf_{h_m}
\sup_{P}
\Prob\left\{
\err(h_m,\Z_u)
\geq \epsilon \right\}\\
\notag
&=\inf_{h_m}
\sup_{P}
\Prob\left\{
\err(h_m,\Z_u) - L(h_m)  + L(h_m)
\geq \epsilon \right\}\\
\notag
&\stackrel{(i)}{\leq}
\inf_{h_m}
\sup_{P}\Bigl[
\Prob\left\{
\err(h_m,\Z_u) - L(h_m) \geq \epsilon/2\right\} 
+ \Prob\left\{L(h_m) \geq \epsilon/2 \right\}\Bigr]\\
\label{eq:proof-relations-II-SSL}
&\stackrel{(ii)}{\leq}
\inf_{h_m}
\sup_{P}
\Prob\left\{
\err(h_m,\Z_u) - L(h_m)  \geq \epsilon/2\right\} 
+ \inf_{h_m}
\sup_{P}\Prob\left\{L(h_m) \geq \epsilon/2 \right\},
\end{align}
where in (i) we used the fact that for any $a,b$, and $\epsilon$ 
if $a + b\geq \epsilon$ then 
either $a \geq \epsilon/2$ or $b\geq \epsilon/2$ holds true and combined it with the union bound $\Prob\{A \cup B\} \leq \Prob\{A\} + \Prob\{B\}$
and (ii) uses $\sup(a+b) \leq \sup a + \sup b$.
Next we write
\begin{align*}
&\inf_{h_m}
\sup_{P}
\Prob\left\{
\err(h_m,\Z_u) - \Prob_{(X,Y)\sim P}\{h_m(X) \neq Y \} 
\geq 
\epsilon/2
\right\}\\
&\leq
\inf_{h_m^0}
\sup_{P}
\Prob\left\{
\err(h_m^0,\Z_u) - \Prob_{(X,Y)\sim P}\{h_m^0(X) \neq Y \} 
\geq 
\epsilon/2
\right\}.
\end{align*}
Since conditioning on $\Z_m$ turns $\err(h^0_m, \Z_u)$ into an average of iid Bernoulli random variables with parameters $L(h_m^0)$,
we use Hoeffding's inequality \cite[Theorem 2.8]{BLM13} and obtain
\begin{align*}
&\Prob\left\{
\err(h_m^0,\Z_u) - \Prob_{(X,Y)\sim P} \{ h_m^0(X) \neq Y\}
\geq \epsilon/2 \right\}\\
&=
\int_{\Z_m}\Prob\left\{
\err(h_m^0,\Z_u) - L(h_m^0)
\geq \epsilon/2  \Big| \Z_m\right\} dP(\Z_m)\\
&\leq
\int_{\Z_m} e^{- u \epsilon^2/2} dP(\Z_m) = e^{- u \epsilon^2/2}.
\end{align*}
Together with \eqref{eq:proof-relations-II-SSL}, this proves the first inequality of \eqref{eq:relations-prob}.
For the second inequality of \eqref{eq:relations-prob}, write
\[
\mathcal{M}^{\mathrm{SSL}}_{\epsilon,m}(\Hyp) = \inf_{h_m} 
\sup_P
\Prob\bigl\{
L( h_m)
\geq \epsilon
\bigr\}
\leq
\inf_{h^0_m} 
\sup_P
\Prob\bigl\{
L( h^0_m)
\geq \epsilon
\bigr\}
=
\mathcal{M}^{\mathrm{SL}}_{\epsilon,m}(\Hyp).
\]

\section{Auxiliary Results}

\begin{lemma}
\label{lemma:binomial-ratio}
Let $n,k,i$ be three non-negative integers such that $i \leq k\leq n$. Then,
\begin{align*}
\frac{{n - i \choose k - i}}{{n \choose k}} 
&\geq
\max\left\{
\left(1 - \frac{n - k}{n-i+1}\right)^i
,
\left(1 - \frac{i}{k+1}\right)^{n-k}
\right\}
\\
&\geq
\exp\left(
-\frac{(n-k)i}{k-i+1}
\right),
\end{align*}
and
\begin{align*}
\frac{{n - i \choose k - i}}{{n \choose k}} 
&\leq
\min\left\{
\left(1 - \frac{n - k}{n}\right)^i
,
\left(1 - \frac{i}{n}\right)^{n-k}
\right\}.
\end{align*}
\end{lemma}
\begin{proof}
To show the first part of the maximum, write
\begin{align*}
\frac{{n - i \choose k - i}}{{n \choose k}} 
=
\frac{(n-i)!k!}{(k-i)!n!}
&= 
\frac{(k - i + 1) \cdots (k-1) k}{(n - i + 1) \cdots (n-1) n}\\
&=
\left(1 - \frac{n - k}{n-i+1}\right)
\left(1 - \frac{n - k}{n-i+2}\right)
\cdots
\left(1 - \frac{n - k}{n}\right)\\
&\geq
\left(1 - \frac{n - k}{n-i+1}\right)^i\\
&\geq
\exp\left(
-\frac{(n-k)i}{k-i+1}
\right),
\end{align*}
where the last inequality follows because $(1 - 1/x)^{x-1}$ monotonically decreases to $e^{-1}$ for $x \geq 1$.

To show the second part of the maximum, write
\begin{align*}
\frac{{n - i \choose k - i}}{{n \choose k}} 
&= 
\frac{(k-i+1)(k-i+2)\cdots(n-i)}{(k+1)(k+2)\cdots n}\\
&=
\left(1 - \frac{i}{k+1}\right)\left(1 - \frac{i}{k+2}\right)\cdots \left(1 - \frac{i}{n}\right)
\geq
\left(1 - \frac{i}{k+1}\right)^{n-k}.
\end{align*}
The upper bounds follow from the same expressions.
\end{proof}
\end{document}